\documentclass{article}

\usepackage[margin=1in]{geometry}
\usepackage{times}
\usepackage[numbers,sort&compress]{natbib}
\usepackage{authblk}

\usepackage{amsmath,amssymb,amsthm,mathtools}
\usepackage{booktabs}
\usepackage{float}
\usepackage{array}
\usepackage{microtype}
\usepackage{multirow}
\usepackage{graphicx}
\usepackage{url}
\usepackage{hyperref}
\hypersetup{hidelinks}
\usepackage{tikz}
\usetikzlibrary{arrows.meta,positioning}
\usepackage{pgfplots}
\pgfplotsset{compat=1.18}

\newtheorem{theorem}{Theorem}[section]
\newtheorem{lemma}[theorem]{Lemma}
\newtheorem{proposition}[theorem]{Proposition}
\newtheorem{corollary}[theorem]{Corollary}
\newtheorem{remark}[theorem]{Remark}

\theoremstyle{definition}
\newtheorem{definition}[theorem]{Definition}

\newcommand{\E}{\mathbb{E}}

\newcommand{\R}{\mathbb{R}}

\newcommand{\Tcal}{\mathcal{T}}

\newcommand{\Ncov}{\mathsf{N}}
\newcommand{\Npack}{\mathsf{P}}

\newcommand{\rd}{\mathsf{R}}
\newcommand{\dist}{\mathsf{d}}

\DeclareMathOperator*{\argmin}{arg\,min}

\title{Curvature as Decoder State:\\
Exact Cache--Update Frontiers for Strongly Convex Unlearning}

\author[1]{Abu Sa-Adat Mohamed Moon-Im Al Ahsan}
\author[2]{Ibne Farabi Shihab}
\author[3]{Md Najmus Swaqeeb}

\affil[1]{Department of Computer Science and Engineering, BRAC University, Bangladesh}
\affil[2]{Department of Computer Science, Iowa State University, USA}
\affil[3]{BRAC University, Bangladesh}

\affil[ ]{\texttt{ishihab@iastate.edu}}

\date{}

\date{}

\begin{document}

\maketitle

\begin{abstract}
Machine unlearning depends on what task-dependent state survives deployment and what a later deletion request reveals. We separate pre-request decoder state $M$ from post-request delivery $B$ and study reconstruction of a fixed retraining procedure under an explicit serialized interface. For smooth strongly convex objectives, parameters plus a revealed deletion gradient leave an exact $d$-dimensional compatible-target ball, while exact curvature on $s$ orthogonal directions removes precisely $s$ dimensions from the high-rate description length. A request-oblivious cache that must hedge $r$ future directions instead preserves $p_r=rd-r(r-1)/2$ symmetric degrees of freedom, distinguishing the state-only and request-time coefficients. Ridge regression realizes these statements through an exact Woodbury downdate: retained rows and exact sufficient statistics require no additional transcript under their stated interfaces, whereas parameters alone require $\Theta(d\log(1/\varepsilon))$ bits on a compact family. A fixed float64 audit covering $270$ cells, $2{,}880$ paired requests, and $17{,}280$ state/interface evaluations confirms that sufficient statistics are dominated by parameters on all nine identifier-only surfaces but become nondominated when requests expose deleted values. Gaussian-design tests recover the predicted $\sqrt{k}$ request scaling and $m^{-3/2}$ curvature fluctuations for retained sample size $m$, while the selected curvature caches in our main comparison reduce delivery by $6.8\%$--$19.6\%$ only by increasing serialized state and requiring repeated requests to amortize it. The resulting frontier prices curvature as decoder state, not free compression, and makes no zero-memory, universal finite-bit, or end-to-end neural-converse claim.
\end{abstract}

\section{Introduction}

Deleting one row from ridge regression exposes the accounting problem sharply. A decoder holding the retained rows can refit without receiving a message; a decoder holding only the current parameter vector may still lack the curvature needed to determine the retained-data solution. Calling both situations ``zero-bit'' or ``one-record deletion'' hides the resource that changes the answer. Figure~\ref{fig:frontier} summarizes the cache--delivery plane: measured systems occupy different operating points, while Theorem~\ref{thm:curvature-hedging} explains why pre-request curvature can have a larger intrinsic dimension than the request-time correction. More generally, one replay item, one parameter, and one optimization step are different units, and retention commits information before a request while retraction may communicate after the request reveals what must change. Any operational claim must therefore state what crosses each interface, when it crosses, the precision of stored quantities, and what the decoder already knows.

Influence functions and Newton-style unlearning make the missing object visible:
their correction is controlled by an inverse Hessian or a close surrogate
\citep{koh2017influence,izzo2021approximate,golatkar2020eternal,bui2024newton,
zhang2024certified}.  That observation is familiar.  Our question is different.
We ask what can be identified after deployment when the decoder has retained only
the fitted parameters, or only a restricted curvature summary, and the request is
revealed later.  The resulting statement is not that curvature helps an algorithm;
it is that, over a declared objective class, omitting curvature leaves a target set
with a calculable dimension and therefore an unavoidable request-time description
length.

We formalize this reduction with a serialized state $S$ of budget $M$ bits and
an update transcript $U$ of budget $B$ bits sent after request $Q$.  The
retention target $Y^+$ is any declared task-dependent object.  The behavioral
retraction target $Y^-$ is the output of a fixed training procedure on the
retained data.  For declared distortions, every protocol obeys
\begin{equation}
 M \geq \rd_{Y^+}(D_+), \qquad
 B \geq \rd_{Y^-\mid S,Q}(D_-), \qquad
 M+B \geq \rd_{Y^-\mid Q}(D_-),
 \label{eq:intro-accounting}
\end{equation}
where the right-hand sides are ordinary or conditional rate--distortion
functions.  Conditional covering numbers give the parallel one-shot statement.
Figure~\ref{fig:two-channel} makes the timing and decoder side information
explicit.

\begin{center}
\fbox{\begin{minipage}{0.94\linewidth}\small
\emph{Reporting tuple.}  Every result in this paper specifies
$(S,Q,\mathsf{codec},D,\mathsf{source})$: the task-dependent decoder state $S$
committed before the request; the information exposed by request $Q$; the exact
serialized representation used for $M$ and $B$; the behavioral distortion $D$ to
the declared retraining target; and the location of any authoritative source
used by the update encoder.  Compute, source-side storage, and certified-erasure
semantics are separate resources and are not silently included in $M$ or $B$.
\end{minipage}}
\end{center}

\begin{figure}[t]
\centering
\begin{tikzpicture}[
  node distance=6mm and 9mm,
  box/.style={draw, rounded corners, align=center, minimum height=7mm,
              inner xsep=5pt, font=\small},
  arr/.style={-{Latex[length=2mm]}, thick},
  lab/.style={font=\scriptsize, align=center}
]
\node[box] (data) {stream or\\training data};
\node[box, right=of data] (state) {decoder cache $S$\\budget $M$ bits};
\node[box, right=of state] (retain) {retained output\\$\widehat Y^+$};
\node[box, below=of state] (request) {deletion request $Q$};
\node[box, right=of request] (update) {update transcript $U$\\budget $B$ bits};
\node[box, right=of update] (retract) {retrained-target output\\$\widehat Y^-$};
\draw[arr] (data) -- node[lab,midway,yshift=24pt] {pre-request\\channel} (state);
\draw[arr] (state) -- (retain);
\draw[arr] (request) -- (update);
\draw[arr] (update) -- node[lab,midway,yshift=-24pt] {post-request\\channel} (retract);
\draw[arr] (state.south east) -- (retract.north west);
\end{tikzpicture}
\caption{The operational distinction.  Retention is limited by the information
carried in the pre-request state.  Retraction is limited by the residual
uncertainty about the retraining target after conditioning on the state and the
request.  Changing what the request reveals changes the bound.}
\label{fig:two-channel}
\end{figure}
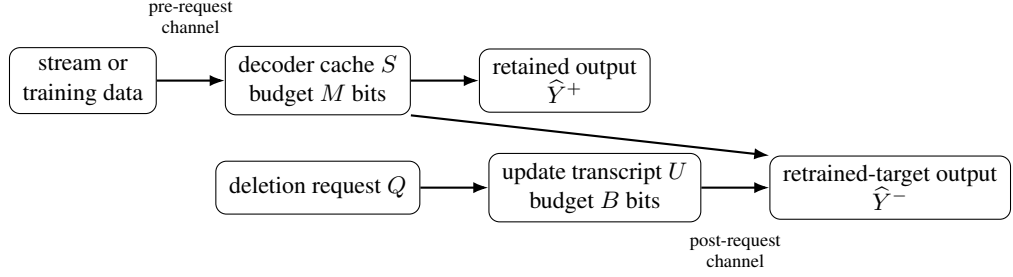

The contribution lies in learner-specific identifiability statements,
not in the coding identities used to state the interface.  First,
Theorem~\ref{thm:sc-ambiguity} gives an exact minimax target geometry for arbitrary
group deletion in the class of $m$-strongly-convex, $L$-smooth objectives:
parameters plus the deleted gradient leave a full $d$-dimensional ball of possible
retraining targets; Corollary~\ref{cor:partial-curvature} identifies the dimensions
removed by exact partial curvature.  Second,
Proposition~\ref{prop:data-native-witness} gives a data-native witness principle for
ridge objectives: reweighting retained examples inside the nullspace of the
full-data first-order condition preserves the deployed model exactly while changing
only the hidden retained curvature relevant to deletion, and the rank of an
explicit Jacobian determines the local metric dimension of the compatible
retraining targets.  Third, Theorem~\ref{thm:curvature-hedging} allows an arbitrary
finite pre-request cache and shows that hedging $r$ possible deletion directions
exposes $p_r=rd-r(r-1)/2$ jointly symmetric degrees of freedom, even though any
one post-request target is only $d$ dimensional.  Finally,
Theorems~\ref{thm:ridge-radius} and~\ref{thm:ridge-hedging} realize the same
separations with exact ridge retraining and explicit finite codes.  Data
processing, conditional covering, the binary caching specialization, the mean-value
identity, and the Woodbury downdate are accounting and proof tools; none is claimed
as a contribution.  The empirical section directly fits the $(d-s)$ coefficient
from Corollary~\ref{cor:partial-curvature}, constructs indistinguishable
natural-data ridge witnesses, audits group-size codec scaling, and then measures
cache--message tradeoffs on nonquadratic convex objectives, frozen deep
representations, and deterministic same-initialization neural retraining.

The proof ingredients are classical: mean-value identities, spectral calculus,
local inverse mapping, and volumetric covering.  We therefore make no claim that
the technical primitives themselves are new.  The substantive claim is the exact
operational consequence of declaring the deployed state: which curvature degrees
of freedom remain indistinguishable, how request timing changes their dimension,
and when that indistinguishability can be witnessed using the observed data rows
themselves.

\section{Related work}\label{sec:related}

Rate--distortion theory supplies the expected-rate and one-shot covering tools
used here \citep{shannon1959coding,berger1971rate,cover2006elements,
wyner1976rate,kostina2012fixed}.  Cache-aided source coding additionally
separates information stored before a request from information delivered after
it \citep{timo2018caching}.  We do not claim a new general coding region.  These
tools serve only to translate a declared unlearning interface into a description
length.  The learner-specific contribution is the exact compatible-target
geometry induced by omitted or partially retained curvature, together with a
ridge family on which request-oblivious placement and post-request delivery have
different leading dimensions.

Machine unlearning ranges from exact deletion and sharded retraining
\citep{cao2015towards,ginart2019making,bourtoule2021machine} to certified and
approximate updates \citep{guo2020certified,izzo2021approximate,
sekhari2021remember,neel2021descent,ullah2021stability,
suriyakumar2022algorithms}.  Ticketed and system-aware formulations price
pre-request storage or stored cores under different interfaces
\citep{ghazi2023ticketed,cherapanamjeri2025space,lu2025systemaware}.
Information-theoretic regularization can instead target independence from
deleted information \citep{xu2026information}.  Our target and resource semantics
are fixed by the reporting tuple in the introduction.

Exact ridge retraction from additive normal-equation statistics is known.
Most directly, \citet{quan2026ridge} maintain ridge-head sufficient statistics
under federated add and delete requests and recover the centralized optimum.
Dynamic least-squares work studies update time under row insertions and
deletions \citep{jiang2023dynamic}.  Our downdate identity is an instrument. The contributions are the exact parameters-only target
ball and the request-oblivious state--update frontier for a compact ridge family;
these price communication instead of dynamic-regression running time.

Second-order information is also central to practical approximate unlearning.
Influence functions express the effect of a training point through inverse-Hessian
actions \citep{koh2017influence}; Fisher and Hessian surrogates have been used to
scrub deep models \citep{golatkar2020eternal}, and recent work studies both
regularized Newton updates and certified inverse-Hessian approximations for neural
networks \citep{bui2024newton,zhang2024certified}.  In particular,
\citet{bui2024newton} study the numerical failure of naive Newton updates under
degenerate neural Hessians and cubic regularization as an algorithmic remedy.  We
do not use a gradient--Fisher--Newton method ranking as evidence for our theorem.
Our nonconvex experiment fixes initialization, minibatch order, and augmentation
randomness, serializes the curvature state available before deletion, and measures
the residual bits needed to reach the same retraining trajectory.

Curvature storage is itself a recognized systems bottleneck.
\citet{sekhari2021remember} retain second-order statistics in a strongly-convex
certified-unlearning construction, whereas \citet{qiao2025hessianfree} replace
explicit Hessian storage and operations with per-example statistical vectors.
Kronecker-factored curvature provides a structured neural state between a diagonal
and a dense Fisher matrix \citep{martens2015kfac,grosse2016kfc}.  Recent minimax results for
smooth strongly-convex losses characterize the statistical price of pure
unlearning in excess risk \citep{regehr2026nearoptimal}.  Those results concern
certification, deletion capacity, optimization, or statistical risk, not
identifiability from a declared deployed state.  Our question is complementary:
after the full-data parameters and request are fixed, what set of exact retraining
targets remains compatible with the decoder's state, and how many bits are required
to resolve it?  Large-scale benchmarks further show that deep unlearning
conclusions are sensitive to the model, request, and baseline
\citep{cadet2024deepunlearn}.  We therefore do not present a competing deep
unlearning algorithm.  The strongly-convex result isolates exact ambiguity,
Proposition~\ref{prop:data-native-witness} shows how the ambiguity arises using
existing natural-data rows, and Theorem~\ref{thm:curvature-hedging} prices
request-oblivious placement.  The neural experiment is only a state-accounting
probe once convexity is removed.

\paragraph{Operational convention.}
All rates below use the interface declared in Figure~\ref{fig:two-channel}.
Task-dependent decoder state is charged to $M$ before the request; the
post-request transcript is charged to $B$; request contents are explicit; and
the authoritative source used by the encoder is outside $M$.  The target, state, request contents, codec, and source location follow the
reporting tuple in the introduction.  Full rate--distortion definitions,
one-shot covering statements, and the request-oblivious binary specialization
are collected in Appendix~\ref{sec:general}.

\section{Curvature identifiability beyond ridge}
\label{sec:curvature-identifiability}

The central ambiguity is not specific to least squares.  Let a training objective
decompose as
\[
F_+(\theta)=F_-(\theta)+L_Q(\theta),
\]
where $F_+$ is the objective before a deletion request, $F_-$ is the retained-data
objective, and $L_Q$ is the aggregate contribution of the deleted group $Q$.

Let $w=\argmin_\theta F_+(\theta)$,
$w^-=\argmin_\theta F_-(\theta)$, and define the deleted gradient at
deployment by
\[
g_Q=\nabla L_Q(w).
\]
The request may contain one record or an arbitrary batch; only its aggregate
gradient enters the following characterization.  We use a gradient-revealing
request interface: after $Q$ is selected, the decoder receives $g_Q$ exactly as
request side information.  The post-request budget $B$ counts only the
additional update transcript sent after this disclosure.  If an implementation
must compute or transmit $g_Q$, its representation must instead be charged to
the request or communication budget; the theorem does not price that cost.

\begin{theorem}[Strongly-convex deletion ambiguity]
\label{thm:sc-ambiguity}
Suppose $F_-$ is twice continuously differentiable and, on the line segment between
$w$ and $w^-$,
\[
mI_d\preceq \nabla^2F_-(\theta)\preceq LI_d
\qquad\text{for }0<m\leq L.
\]
Then
\begin{equation}
w^--w=\bar H_Q^{-1}g_Q,\qquad
\bar H_Q=\int_0^1\nabla^2F_-\!\left(w+t(w^--w)\right)\,dt,
\label{eq:mean-hessian-update}
\end{equation}
with $mI_d\preceq\bar H_Q\preceq LI_d$.  Hence every compatible displacement lies
in
\begin{equation}
\overline B\!\left(
\frac{m^{-1}+L^{-1}}{2}\,g_Q,\,
\frac{m^{-1}-L^{-1}}{2}\,\lVert g_Q\rVert_2
\right).
\label{eq:sc-ambiguity-ball}
\end{equation}
Conversely, fix $(w,g_Q,m,L)$ with $g_Q\neq0$.  Over the quadratic retained
objectives
\[
F_{-,H}(\theta)
=\frac12(\theta-w)^\top H(\theta-w)-g_Q^\top(\theta-w),
\qquad mI_d\preceq H\preceq LI_d,
\]
and the affine deleted contribution
$L_Q(\theta)=g_Q^\top(\theta-w)$, the set of retraining targets is exactly the ball
in~\eqref{eq:sc-ambiguity-ball}.  Writing
\[
\rho_Q=\frac{m^{-1}-L^{-1}}{2}\lVert g_Q\rVert_2,
\]
a decoder whose pre-request task-dependent state contains only $w$, with
$(m,L)$ public and $g_Q$ revealed by the request, therefore obeys, for
$0<\varepsilon<\rho_Q$,
\begin{equation}
d\log_2\frac{\rho_Q}{\varepsilon}
\leq B^\star(\varepsilon)
\leq
\left\lceil d\log_2\left(1+\frac{2\rho_Q}{\varepsilon}\right)\right\rceil
\label{eq:sc-ambiguity-rate}
\end{equation}
on this objective class.
\end{theorem}

Equation~\eqref{eq:mean-hessian-update} is the finite-deletion analogue of the
inverse-Hessian expression motivating influence functions, but it makes a different
statement.  The matrix is an average retained-data Hessian along the entire path to
retraining, not merely the Hessian evaluated at $w$.  If that object is absent from
decoder state, strong convexity and smoothness alone leave a full-dimensional
compatible target set.  The result therefore applies to group deletion and to
nonquadratic objectives without assuming that a one-step Newton approximation is
exact.  The quadratic subclass makes the minimax ball tight.

A practical decoder rarely stores an unrestricted dense Hessian.  The next result
prices a clean partial-curvature interface and exposes the dimension that remains
unresolved.

\begin{corollary}[Partial-curvature cache]
\label{cor:partial-curvature}
Let $P$ be a public orthogonal projector of rank $s<d$, let
$P_\perp=I-P$, and fix a positive-definite operator $H_\parallel$ on
$\operatorname{range}(P)$ satisfying
$mP\preceq H_\parallel\preceq LP$.  Consider the quadratic subclass in
Theorem~\ref{thm:sc-ambiguity} with block-diagonal curvature
\[
H=H_\parallel\oplus H_\perp,\qquad
mP_\perp\preceq H_\perp\preceq LP_\perp,
\]
and suppose the decoder state reveals $(w,g_Q,P,H_\parallel)$ exactly.  Then the
resolved component of the deletion displacement is
$H_\parallel^{-1}Pg_Q$, while the unresolved component fills exactly the
$(d-s)$-dimensional ball in $\operatorname{range}(P_\perp)$ centered at
\[
\frac{m^{-1}+L^{-1}}{2}P_\perp g_Q
\]
with radius
\[
\rho_{\perp}
=\frac{m^{-1}-L^{-1}}{2}\lVert P_\perp g_Q\rVert_2.
\]
If $\rho_\perp>0$ and $0<\varepsilon<\rho_\perp$, the minimum worst-case
request-time description length satisfies
\begin{equation}
(d-s)\log_2\frac{\rho_\perp}{\varepsilon}
\leq B_s^\star(\varepsilon)
\leq
\left\lceil(d-s)\log_2\left(1+\frac{2\rho_\perp}{\varepsilon}\right)\right\rceil .
\label{eq:partial-curvature-rate}
\end{equation}
\end{corollary}

Corollary~\ref{cor:partial-curvature} gives a literal interpretation of a low-rank
curvature cache.  Resolved directions disappear from the request-time metric
entropy; unresolved directions do not.  No claim is made that the best practical
deep-network curvature summary must be an orthogonal block.  The point is that
partial second-order state admits a sharp information value even before choosing a
numerical unlearning algorithm.  Section~\ref{sec:verification} tests the same
monotone relation with diagonal and low-rank curvature summaries on natural and
deep-feature objectives.

The block construction above makes the unresolved dimension explicit.  The same
identifiability phenomenon can also be witnessed using fixed rows and labels from
an ordinary ridge dataset.

Appendix Proposition~\ref{prop:data-native-witness} gives a local data-native
witness for the same ambiguity: sufficiently small retained-example
reweightings with positive final weights preserve the deployed ridge optimum
while changing the hidden retained curvature and exact deletion target.  Its
explicit target Jacobian sets the local metric dimension, yielding the full
$\Omega(d\log(1/\varepsilon))$ coefficient only when that Jacobian has full
rank.

Proposition~\ref{prop:data-native-witness} separates two facts that are otherwise
easy to conflate.  Newton-style dependence on curvature is familiar; the
proposition states when that missing curvature creates an actual
indistinguishability family under a fixed deployed model and fixed deletion
request.  The empirical witness in Section~\ref{sec:natural-witness} evaluates
the rank and finite packing of this family on public regression data.

The preceding corollary fixes the subspace in advance.  A stronger deployment
question is whether an arbitrary finite cache can hedge several deletion
directions before the request is known.  The answer depends on the dimension of
the symmetric operator that those requests jointly interrogate.

\begin{theorem}[Request-oblivious curvature frontier]
\label{thm:curvature-hedging}
Fix $1\leq r\leq d$, $a>0$, and $0<m<L$.  Let
\[
h_1=L^{-1},\qquad h_2=m^{-1},\qquad
\delta=\frac{h_2-h_1}{2},\qquad
p_r=rd-\frac{r(r-1)}2.
\]
For every symmetric $H$ with $h_1I_d\preceq H\preceq h_2I_d$, set
$A=H^{-1}$ and define the full objective
\[
F_H^+(\theta)=\frac12(\theta-w)^\top A(\theta-w).
\]
Let $e_q$ denote the $q$th standard basis vector,
$L_q(\theta)=a e_q^\top(\theta-w)$ for $q\in[r]$, and
$R_H(\theta)=F_H^+(\theta)-\sum_{j=1}^rL_j(\theta)$.  Then
$F_H^+=R_H+\sum_jL_j$ is a decomposition into $r$ public deletable affine
contributions and a strongly-convex remainder.  If request $q$ removes $L_q$,
the exact retraining target is
\[
w_q^-=w+aHe_q.
\]
Let an arbitrary fixed-length state of at most $M$ bits be formed before $q$ is
known and let a request-specific transcript of at most $B$ bits produce
$\widehat w_q$ satisfying
$\lVert\widehat w_q-w_q^-\rVert_2\leq\varepsilon$ for every admissible $H$ and
every $q\in[r]$.  For
$0<\varepsilon<\delta a/(2\sqrt{2r})$,
\begin{align}
M+rB
&\geq
p_r\log_2\frac{\delta a}{2\sqrt{2r}\,\varepsilon},
\label{eq:curvature-hedging-joint}\\
M+B
&\geq
d\log_2\frac{\delta a}{\varepsilon}
\qquad\text{for every fixed request.}
\label{eq:curvature-hedging-single}
\end{align}
Conversely, for every $\alpha\in[0,1]$ there is a fixed-length code with
\begin{align}
M&\leq \alpha p_r\log_2(1/\varepsilon)+O(p_r\log d),\\
B&\leq (1-\alpha)d\log_2(1/\varepsilon)+O(d\log d),
\end{align}
where constants depend only on $(m,L,a)$.  Separate endpoint codes attain
exactly $M=0$ at $\alpha=0$ and exactly $B=0$ at $\alpha=1$.  Hence the
state-only endpoint has leading coefficient $p_r$, while the message-only
endpoint has leading coefficient $d$.
\end{theorem}

The dimension $p_r$ is the number of independent entries in the first $r$
columns of a symmetric $d\times d$ operator.  A request-time message needs only
one $d$-dimensional column, whereas a request-oblivious state that must answer
all $r$ possibilities has to preserve their jointly symmetric degrees of
freedom.  At $r=d$, this becomes $d(d+1)/2$ versus $d$.  The ridge construction
in Theorem~\ref{thm:ridge-hedging} realizes the same separation with actual
squared-loss records instead of affine objective components.

Appendix Theorem~\ref{thm:gaussian-ridge} supplies the typical-design
correction law absent from the worst-case geometry.  It is a fixed-$d$,
$m\to\infty$ result with a request score independent of the retained design:
the curvature fluctuation has $m^{-3/2}$ scale, and known population covariance
yields the stated reverse-water-filling benchmark.  Covariance uncertainty
instead leaves a shrinking $d$-dimensional ball of possible leading means.

Appendix Corollary~\ref{cor:identifier-gaussian} isolates the
identifier-only interface when retained curvature is exact but deleted scores
are hidden.  Under finite-moment assumptions, its score-sum correction is a
multivariate central limit theorem (CLT); the finite-$k$ Gaussian law is exact only when the score
vectors themselves are Gaussian.  Otherwise the reverse-water-filling law is
an asymptotic large-$k$ benchmark.

Theorem~\ref{thm:gaussian-ridge} supplies a typical-design scale that the
worst-case balls do not.  If a random group request has
$v_Q=\sum_{i\in Q}\xi_i$ with independent mean-zero score vectors of finite
second moment, then $\|v_Q\|_2=O_p(\sqrt{k})$.  Both the population-curvature
radius and the first-order Wishart fluctuation scale inherit this $\sqrt{k}$
factor, so increasing $k$ from $1$ to $32$ raises a fixed-step high-rate
coordinate width by approximately $\tfrac12\log_2 32=2.5$ bits instead of
$\log_2 32=5$ bits.  Section~\ref{sec:gaussian-typical-audit} tests this scaling
with the high-sample ridge sweep; the three-seed neural diagnostic below is not
used to estimate a $k$ exponent.

Appendix Remark~\ref{rem:rank-r-dimension} prevents conflating the target
dimension of one request with the operator slice needed before a request is
known.  One realized rank-$r$ deletion still has a $d$-dimensional message
target, whereas the $p_r=rd-r(r-1)/2$ quantity is the pre-request operator-slice
dimension for hedging an $r$-dimensional family of future directions.

\section{Ridge realization and empirical tests}
\label{sec:main-evidence}

The general coding statements and the complete ridge constructions are given in
Appendices~\ref{sec:general} and~\ref{sec:retraction}.  The interface is the one
declared in the introduction: $M$ prices task-dependent state committed before
the deletion request and $B$ prices the transcript delivered after the request.
The standard data-processing and covering arguments are used only to convert a
compatible-target set into bits.

Figure~\ref{fig:frontier} places the theory and measured systems on the same
cache--delivery plane.  The theoretical curve uses the ridge family with $d=90$, $\varepsilon=10^{-4}$,
and normalized target-ball radius $\rho=1$; state sizes use the declared fp32 layouts, while nonzero delivery
coordinates are approximate means reconstructed from rounded table entries.  Zero-delivery methods
are drawn at the plotting floor and labeled $B=0$.

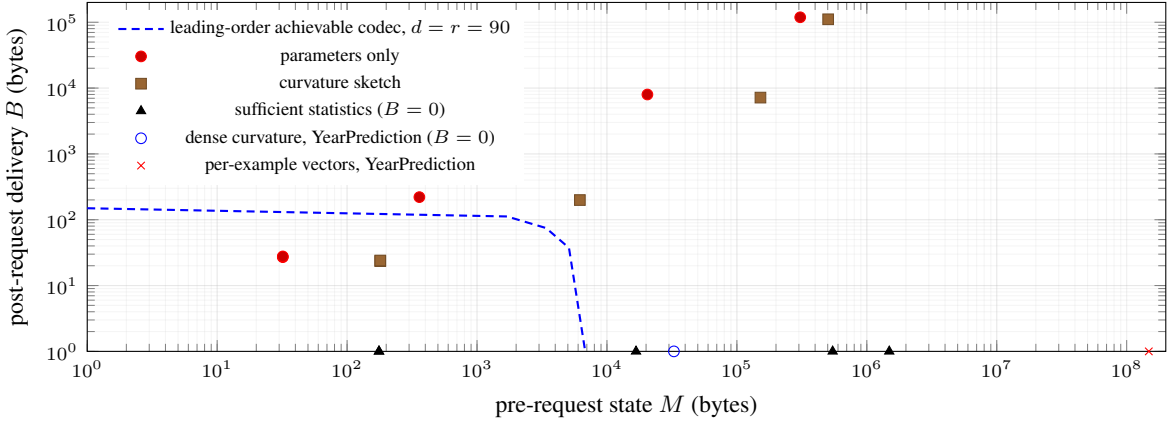
\begin{figure}[t]
\centering
\begin{tikzpicture}
\begin{loglogaxis}[
width=0.96\linewidth,height=6.2cm,
xlabel={pre-request state $M$ (bytes)},
ylabel={post-request delivery $B$ (bytes)},
xmin=1,xmax=2e8,ymin=1,ymax=2e5,
grid=both,
minor grid style={opacity=0.15},
major grid style={opacity=0.25},
legend style={
    font=\scriptsize,
    at={(0.02,0.98)},
    anchor=north west,
    draw=none
},
tick label style={font=\scriptsize},
label style={font=\small}
]

\addplot+[mark=none,thick,densely dashed] coordinates {
(1,149.49)
(1700.41,112.12)
(3400.82,74.74)
(5101.24,37.37)
(6801.65,1)
};
\addlegendentry{leading-order achievable codec, $d=r=90$}

\addplot+[only marks,mark=*] coordinates {
(32,27.0)
(32,27.7)
(360,220.5)
(20480,7987.2)
(307200,118988.8)
};
\addlegendentry{parameters only}

\addplot+[only marks,mark=square*] coordinates {
(180,23.5)
(180,24.0)
(6188,199.1)
(151812,7168)
(504068,110899.2)
};
\addlegendentry{curvature sketch}

\addplot+[only marks,mark=triangle*] coordinates {
(176,1)
(176,1)
(16740,1)
(545792,1)
(1488384,1)
};
\addlegendentry{sufficient statistics ($B=0$)}


\addplot+[only marks,mark=o] coordinates {
(32760,1)
};
\addlegendentry{dense curvature, YearPrediction ($B=0$)}

\addplot+[only marks,mark=x] coordinates {
(148000000,1)
};
\addlegendentry{per-example vectors, YearPrediction}

\end{loglogaxis}
\end{tikzpicture}

\caption{
Cache--delivery operating points. The dashed curve shows the
leading-order achievable state--update interpolation for the request-hedging
ridge family with $d=r=90$, $p=d(d+1)/2$, $\varepsilon=10^{-4}$,
and normalized radius $\rho=1$; its ideal bit rates are divided by eight
to display bytes, and lower-order and serialization costs are omitted. The $M=0$ and $B=0$ endpoints are displayed
at the one-byte plotting floor for log-scale visibility. Filled markers show
approximate operating points for parameters, low-rank curvature, and packed
fp32 sufficient statistics, ordered as California Housing, Concrete,
YearPredictionMSD, CIFAR-10 ResNet features, and CIFAR-100 vision transformer (ViT) features.
The first two sufficient-statistic points coincide. Nonzero delivery
coordinates are reconstructed from rounded table means. Empirical points use
$k=32$ value-exposing requests and their table-specific distortion criteria;
points from different tasks should not be interpreted as a common Pareto
frontier. }

\label{fig:frontier}
\end{figure}

Published unlearning methods occupy interpretable locations on this plane once
their retained object is made explicit.  Table~\ref{tab:published-frontier}
uses our codec and the same ridge/deep-feature tasks; the mapping is an accounting
instantiation, not a claim that the original papers optimize our communication
objective.  Exact ridge sufficient statistics follow \citet{quan2026ridge};
certified-removal and generalization-oriented second-order methods motivate a
dense curvature state \citep{guo2020certified,sekhari2021remember}; source-side
influence updates leave curvature outside the decoder state
\citep{koh2017influence,izzo2021approximate}; Hessian-free online unlearning stores

per-example statistical vectors \citep{qiao2025hessianfree}.
\begin{table}[t]
\centering
\caption{Published methods under the paper's decoder-state accounting.  The
state column is the object charged when the method is instantiated as a deployed
decoder on the same tasks; $B$ is the additional model-update packet under our
codec.  ``Local'' means the method can complete from cached state and request
without a model-update transcript, while its native compute/noise guarantees
remain those of the cited method.}
\label{tab:published-frontier}
\footnotesize
\setlength{\tabcolsep}{2.5pt}
\begin{tabular}{@{}p{2.65cm}p{3.65cm}p{2.25cm}p{2.45cm}@{}}
\toprule
Method family & decoder state charged to $M$ & request-time placement & YearPredictionMSD instantiation \\
\midrule
Exact ridge sufficient statistics \citep{quan2026ridge} & packed $(G,c)$ & $B=0$ & $M=16.35$ KiB, $B=0$ \\
Dense-Hessian certified removal \citep{guo2020certified,sekhari2021remember} & parameters + dense $d\times d$ curvature & local correction + native noise term & $M=31.99$ KiB; exact ridge algebra gives $B=0$ before any certification noise \\
Influence / approximate deletion \citep{koh2017influence,izzo2021approximate} & parameters if Hessian remains source-side & one $d$-vector correction & $M=0.352$ KiB, $B=220.5$ bytes under our residual codec \\
Hessian-free online \citep{qiao2025hessianfree} & one $d$-vector statistic per training record & local vector aggregation & $M\approx141$ MiB for the training split, $B\approx0$ \\

\bottomrule
\end{tabular}
\end{table}

Ridge provides an exact squared-loss realization of the curvature ambiguity.
Appendix Theorem~\ref{thm:ridge-radius} constructs instances with the same
deployed $w$ and the same value-exposing singleton request whose exact retraining
targets fill a $d$-dimensional ball, giving
$B_{\rm param}^\star(\varepsilon)=\Theta(d\log(1/\varepsilon))$.
Appendix Theorem~\ref{thm:ridge-hedging} then exposes the request-timing
separation.  For $d$ possible public deletion directions, a state that must answer
all future requests has a state-only leading coefficient
$p=d(d+1)/2$, whereas after the request only a $d$-dimensional correction must be
delivered.  Exact sufficient statistics reach $B=0$ under a value-exposing
request, but their finite serialized cost is charged to $M$.

The first empirical test targets Corollary~\ref{cor:partial-curvature} directly.
At $d=64$, we cache exact curvature on coordinate subspaces of
rank $s$ and fit the request-time payload against
$\log_2(1/\varepsilon)$.  Table~\ref{tab:main-partial-slope} reports the
predicted and measured coefficients.  The largest relative slope discrepancy is
\ensuremath{0.8\%}; all \ensuremath{18{,}000} reconstructions meet
their distortion target.  Because this audit instantiates the corollary's
block-diagonal family and its explicit unresolved-subspace codec, it validates
the achievability coefficient and implementation, not the volumetric converse.
The natural-data witness below instead tests the lower-bound premise outside the
spectral construction.

\begin{table}[t]
\centering
\caption{Direct test of the partial-curvature coefficient. The $\pm$ values are ordinary least-squares slope standard errors across nine codec operating points.}
\label{tab:main-partial-slope}
\small
\setlength{\tabcolsep}{7pt}
\begin{tabular}{@{}rrrr@{}}
\toprule
$s$ & predicted $d-s$ & fitted slope & fitted/predicted \\
\midrule
\ensuremath{0}  & \ensuremath{64} &
\ensuremath{64.41\pm0.63} & \ensuremath{1.006} \\
\ensuremath{16} & \ensuremath{48} &
\ensuremath{47.62\pm0.48} & \ensuremath{0.992} \\
\ensuremath{32} & \ensuremath{32} &
\ensuremath{32.27\pm0.37} & \ensuremath{1.008} \\
\ensuremath{48} & \ensuremath{16} &
\ensuremath{15.88\pm0.21} & \ensuremath{0.993} \\
\bottomrule
\end{tabular}
\end{table}

The second test removes the hand-built spectral family.  On each public
regression split we reweight only retained examples inside the nullspace of the
full-data first-order condition from
Proposition~\ref{prop:data-native-witness}.  The rows, labels, deployed $w$, and
deleted record remain fixed.  Across \ensuremath{250} requests per
dataset, the target Jacobian is full rank for every request; the maximum change
in the full-data solution is below \ensuremath{8\times10^{-13}}.  At
parameter RMSE $\varepsilon=10^{-4}$, finite target packings
require at least \ensuremath{11.35},
\ensuremath{11.50}, and \ensuremath{12.53} bits on California
Housing, Concrete, and YearPredictionMSD, respectively.  Because each packing is
formed from only \ensuremath{8{,}192} candidates, these finite counts are
capped at \ensuremath{13} bits and are used only as nonzero sanity checks.
The lower-bound evidence is instead the full-rank, locally conditioned target
Jacobian required by Proposition~\ref{prop:data-native-witness}.

Practical cache sweeps are summarized in
Table~\ref{tab:main-direction-counting}.  The natural Hessians are not block
diagonal, and the update packet retains every parameter coordinate, so $s/d$ is
a geometric reference instead of a predicted payload reduction.  The observed
gains are smaller and dataset dependent.  On the tabular ridge tasks, the
pre-quantization residual magnitude grows approximately as $\sqrt{k}$ under
random deletion subsets: fitting the median maximum residual index to
$k^{\alpha}$ gives
$\alpha=\ensuremath{0.47}$,
$\ensuremath{0.51}$, and
$\ensuremath{0.49}$ on California Housing, Concrete, and
YearPredictionMSD.  The corresponding increase from $k=1$ to $k=32$ is
\ensuremath{2.4},
\ensuremath{2.6}, and
\ensuremath{2.5} bits per parameter coordinate, close to the
$\tfrac12\log_2 32=\ensuremath{2.5}$-bit scaling implied by a
mean-zero first-order sum.  Full codec payloads, state sizes, distortion checks,
and group-size sweeps are reported in Appendix~\ref{sec:verification}.

\begin{table}[t]
\centering
\caption{Natural-objective cache sweep.  ``Reduction'' is the measured decrease
in total post-request packet bytes relative to parameters only at the same
distortion.  The $s/d$ column is the exact coefficient only for the orthogonal
block interface of Corollary~\ref{cor:partial-curvature}; it is shown here as a
reference, not as a prediction.}
\label{tab:main-direction-counting}
\small
\setlength{\tabcolsep}{4pt}
\begin{tabular}{@{}lrrrr@{}}
\toprule
Objective & $d$ & cache rank $s$ & $s/d$ & reduction \\
\midrule
California ridge & $8$ & \ensuremath{4} &
\ensuremath{50.0\%} & \ensuremath{13.0\%} \\
Concrete ridge & $8$ & \ensuremath{4} &
\ensuremath{50.0\%} & \ensuremath{13.4\%} \\
YearPrediction ridge & $90$ & \ensuremath{16} &
\ensuremath{17.8\%} & \ensuremath{9.7\%} \\
MNIST logistic & $784$ & \ensuremath{128} &
\ensuremath{16.3\%} & \ensuremath{12.2\%} \\
ResNet-feature logistic & $512$ & \ensuremath{128} &
\ensuremath{25.0\%} & \ensuremath{19.6\%} \\
CIFAR-10 ResNet ridge head & $512$ & \ensuremath{64} &
\ensuremath{12.5\%} & \ensuremath{10.3\%} \\
CIFAR-100 ViT ridge head & $768$ & \ensuremath{64} &
\ensuremath{8.3\%} & \ensuremath{6.8\%} \\
\bottomrule
\end{tabular}
\end{table}

Finally, the nonconvex audit does not compare independently trained neural
parameters.  Deletion retraining restarts from the exact original
initialization, retains the same example order, and uses deterministic
example-indexed augmentation.  The decoder stores parameters alone, parameters
plus diagonal Fisher, or parameters plus untruncated layerwise Kronecker-factored approximate curvature (K-FAC) factors;
the encoder sends only the canonical-coordinate residual needed to reach the
same retraining target at per-coordinate root mean squared error (RMSE)
\ensuremath{10^{-4}}. Slot-preserving minibatches keep retained
examples in the same training batches after deletion, and the validation-selected
quantization steps are
\ensuremath{2.9\times10^{-4}} and
\ensuremath{3.0\times10^{-4}}.  For the Fashion-MNIST MLP,
K-FAC changes $(M,B)$ from
\ensuremath{(0.90,0.280)} MiB to
\ensuremath{(2.39,0.252)} MiB; for ResNet-18 it changes
\ensuremath{(42.66,13.34)} MiB to
\ensuremath{(222.10,12.01)} MiB.  For a horizon of $T$ deletion requests, the cumulative state-plus-delivery
cost is $M+TB$. The corresponding crossover horizons are
\ensuremath{54} and \ensuremath{135} requests.
This experiment therefore remains inside the paper's state/message accounting
and does not duplicate the Newton-versus-cubic algorithm ranking studied by
\citet{bui2024newton}.  Table~\ref{tab:nonconvex-controls} verifies exact
zero-deletion replay under the declared training protocol and records the
quantization step and modal 10/9-bit body widths that determine the reported
serialized payloads.  Equation~\eqref{eq:neural-residual-range} makes the
corresponding canonical-coordinate dynamic ranges explicit.  We do not use the
three-seed neural audit to estimate a deletion-size exponent or to test the
$\sqrt{k}$ law; that scaling test is restricted to the higher-powered
ridge/Gaussian experiments.

The Gaussian-design audit tests Theorem~\ref{thm:gaussian-ridge} instead of an
explicit codec construction.  Across covariance conditions, the measured
request-score norm scales as $k^{0.50\pm0.03}$, the
population-curvature radius as $m^{-1.01\pm0.03}$, and the
known-$\Sigma$ fluctuation standard deviation as
$m^{-1.49\pm0.04}$.  Reverse-water-filling predictions are
within 0.38 bits per active coordinate of empirical Gaussian
code lengths.  This is the first experiment in the paper whose quantitative
scaling is predicted by a distributional converse instead of fixed by the
achievability codec.

A fixed-downprojection adapter gives a direct deployment rule.  With a frozen
backbone and public projection $A\in\mathbb R^{r\times H}$, training only the
output factor $W_{\mathrm{out}}$ under squared loss is multivariate ridge in the $r$ projected
features.  On CIFAR-100 ViT-B/16 features with
$r=16$, the trainable adapter is
6.25 KiB and its exact sufficient statistics are
6.78 KiB; one parameters-only deletion packet is
3.34 KiB at the declared distortion, so sufficient statistics
amortize after the first repeated request.  The practical decision is therefore:
store $(G,c)$ whenever its serialized state fits; use a parameters-only
$d$-vector delivery when it does not; use a sketch only when the expected
request horizon exceeds its measured crossover.
Additional finite-precision validation, extended empirical controls, and
complete proofs of the main converse and ridge results are provided in
Appendices~\ref{sec:verification}, \ref{sec:general}, and
\ref{sec:retraction}.

\section{Discussion, Limitations, and Conclusion}
\label{sec:discussion}

Our results separate information availability from algorithmic efficiency: exact sufficient statistics eliminate post-request delivery when their curvature and response state is retained and the request reveals deleted values, whereas under identifier-only requests the retained state remains dominated by parameters on all nine measured surfaces. For strongly convex smooth objectives, parameters plus the deleted gradient leave an exact full-dimensional worst-case target ball, while resolving an $s$-dimensional curvature block removes exactly $s$ dimensions from the high-rate description length; when the future request is unknown, hedging multiple deletion directions instead requires a larger symmetric operator slice. Ridge regression realizes these statements with exact retraining, and optimum-preserving reweightings provide data-native witnesses using fixed observed rows and labels, although the global minimax construction remains quadratic and worst case and the reweighting argument does not claim that ordinary unweighted i.i.d.\ samples realize the full ambiguity ball. The Gaussian audit supports the predicted $\sqrt{k}$ request growth,
$m^{-3/2}$ curvature fluctuations, and $m^{-1}$ mean-radius term; its
reverse-water-filling result assumes request scores independent of the retained
design, and ridge residuals approach this interface only asymptotically. Partial-curvature experiments recover the $(d-s)$ coefficient on the orthogonal block interface, whereas natural Hessians couple cached and uncached directions, so practical delivery reductions need not follow the ideal $s/d$ law. Across the selected natural low-rank and neural curvature configurations in Table~\ref{tab:main-direction-counting}, request-time delivery decreases by approximately \ensuremath{6.8\%}--\ensuremath{19.6\%}, but serialized state grows by roughly \ensuremath{1.6\times}--\ensuremath{129\times} relative to parameters alone, yielding cumulative state-plus-delivery crossovers of about \ensuremath{25}--\ensuremath{2{,}181} requests; exact For group requests, a realized rank-$r$ deletion still has a $d$-dimensional
message target, whereas $p_r=dr-r(r-1)/2$ measures the pre-request operator
slice needed to hedge future rank-$r$ directions, not a one-shot message
dimension. The nonconvex neural audit shows operational state--delivery separation under
fixed training randomness but does not extend the convex converse to end-to-end
neural retraining. Finite witness packings lower-bound only the sampled families,
and the local asymptotic coefficient depends on the measured Jacobian rank. Because the authoritative source lies outside $M$, our accounting compares
retained deployment state with post-request communication rather than total
systems cost, which also includes source storage, compute, latency, and privacy. The resulting frontier compares declared deployment interfaces and request
horizons rather than universally ranking unlearning algorithms. Cost claims must
state what data-derived state survives deployment, what the request reveals,
and what may be communicated afterward.
\IfFileExists{iclr2027_conference.bst}{%
  \bibliographystyle{iclr2027_conference}
}{%
  \bibliographystyle{plainnat}
}\section*{Reproducibility statement}

Complete assumptions and proofs for the strongly-convex ambiguity set, partial-curvature rate, optimum-preserving reweighting, ridge target ball, and request-hedging converses appear in Sections~\ref{sec:curvature-identifiability}--\ref{sec:main-evidence} and the appendices. All experiments freeze decoder or feature state before sampling deletion requests, use direct retraining as the behavioral reference, and account separately for retained state, requests, and serialized update payloads. Hyperparameters and quantization scales are selected on validation data before evaluation, and residual codecs use canonical parameter coordinates with a fixed \ensuremath{18}-byte packet header. The partial-curvature, natural-data, ridge, spectral, Gaussian-design, and nonconvex audits use fixed protocols, deterministic splits or seeds where
applicable, and validation-selected hyperparameters fixed before test
evaluation; Gaussian request scores are sampled independently of retained designs, and nonconvex retraining restarts from the exact original initialization with deterministic data and augmentation randomness. Run manifests record seeds, dataset splits, codec settings, serialized packet
lengths, validation-selected hyperparameters, and per-request measurements
needed to reconstruct the reported tables and Figure~\ref{fig:frontier}. Reported slope uncertainties are ordinary least-squares fit errors over codec operating points, while other $\pm$ terms are request- or split-level standard deviations as specified in the captions. Experiments were run using [ACTUAL CPU/GPU HARDWARE] with
[ACTUAL RAM/VRAM], using Python [VERSION] and the principal numerical
and learning libraries described in the released environment; the complete
software environment and dependency versions will be provided with the code. Code, experiment runners, manifests, codec implementations, and raw per-request logs will be made publicly available upon acceptance.

\section*{AI use disclosure}

Generative AI tools assisted literature discovery, prose organization,
feedback on mathematical claims and proof structure, and pre-submission
checks of numerical consistency, references, cross-references, and
formatting. The authors reviewed the AI-assisted material and take
responsibility for the final text, mathematical claims, citations,
protocols, and reported results.

\section*{Ethics statement}

This work uses public benchmark datasets and pretrained model
representations and does not collect new human-subject data. The
information-theoretic and communication guarantees studied here are
technical properties of the declared interfaces and should not be
interpreted as legal certification of data deletion, erasure, or
compliance with any particular privacy regulation. Any deployment remains
subject to the licenses, data-governance requirements, and privacy
obligations associated with the underlying datasets and models.

\bibliography{information_accounting_refs}
\clearpage
\appendix

\section{Extended formal statements}
\label{app:extended-formal}

\begin{proposition}[Data-native ridge ambiguity from optimum-preserving reweighting]
\label{prop:data-native-witness}
Let $(x_i,y_i)_{i=1}^n$ be fixed records, let $\lambda>0$, and let $w$ be the
unique minimizer of
\[
F^+(\theta)=\frac12\sum_{i=1}^n(x_i^\top\theta-y_i)^2
+\frac{\lambda}{2}\|\theta\|_2^2.
\]
Fix a deletion set $Q$ and write $R=[n]\setminus Q$,
$r_i=x_i^\top w-y_i$, and
\[
\mathcal V_Q
=
\left\{
v\in\mathbb R^n:
v_i=0\ (i\in Q),\
\sum_{i\in R}v_i x_i r_i=0
\right\}.
\]
For every sufficiently small $v\in\mathcal V_Q$ with $1+v_i>0$, define the
weighted full objective
\[
F_v^+(\theta)
=
\frac12\sum_{i=1}^n(1+v_i)(x_i^\top\theta-y_i)^2
+\frac{\lambda}{2}\|\theta\|_2^2.
\]
Then $w$ is the unique full-data minimizer for every such $v$, the deleted
records and their values are identical, and the exact retraining target after
deleting $Q$ is
\begin{equation}
T_Q(v)
=
w+H_-(v)^{-1}g_Q,
\qquad
H_-(v)=\lambda I+\sum_{i\in R}(1+v_i)x_ix_i^\top,
\qquad
g_Q=\sum_{i\in Q}x_i r_i.
\label{eq:data-native-target}
\end{equation}
Let $V\subseteq\mathcal V_Q$ be a $q$-dimensional linear subspace and define
\begin{equation}
\mathcal J_Q[v]
=
-H_-(0)^{-1}
\left(\sum_{i\in R}v_i x_ix_i^\top\right)
H_-(0)^{-1}g_Q .
\label{eq:data-native-jacobian}
\end{equation}
If $\mathcal J_Q|_V$ has rank $q$, then there exist constants
$c_0>0$ and $\varepsilon_0>0$ such that a decoder whose task-dependent state
contains only the common deployed parameter $w$, the fixed request contents, and
public bounds on the weighting neighborhood must satisfy
\begin{equation}
B^\star(\varepsilon)
\geq
q\log_2\frac{c_0}{\varepsilon}
\qquad
(0<\varepsilon<\varepsilon_0).
\label{eq:data-native-lower-bound}
\end{equation}
Thus a full-rank target Jacobian yields the same
$\Omega(d\log(1/\varepsilon))$ local description-length coefficient using the
observed rows and labels, without constructing a spectral Hessian family.
\end{proposition}

\begin{theorem}[Gaussian-design conditional correction law]
\label{thm:gaussian-ridge}
Let $m=n-k$, let retained feature rows $x_i\in\mathbb R^d$ be i.i.d.\
$\mathcal N(0,\Sigma)$ with $\Sigma\succ0$, and define
\[
W_m=\sum_{i=1}^{m}x_ix_i^\top,\qquad
A_0=m\Sigma+\lambda I_d.
\]
Fix a nonzero request score vector $v_Q$ that is measurable from the request and
public deployed state and independent of the retained design.  The exact
curvature correction under this interface is
$\Delta_m=(W_m+\lambda I_d)^{-1}v_Q$.  Put
$u_m=A_0^{-1}v_Q$ and define the fixed limiting covariance
\begin{equation}
\mathcal V(v_Q,\Sigma)
=(v_Q^\top\Sigma^{-1}v_Q)\Sigma^{-1}
+\Sigma^{-1}v_Qv_Q^\top\Sigma^{-1}.
\label{eq:gaussian-limit-covariance}
\end{equation}
For fixed $d$ and $m\to\infty$,
\begin{equation}
m^{3/2}(\Delta_m-u_m)
\Rightarrow
\mathcal N\!\left(0,\mathcal V(v_Q,\Sigma)\right).
\label{eq:gaussian-delta}
\end{equation}
Equivalently, the first-order covariance scale is
\begin{equation}
\overline V_m=m^{-3}\mathcal V(v_Q,\Sigma),
\qquad
\operatorname{Cov}(\Delta_m\mid v_Q)=\overline V_m+o(m^{-3}).
\label{eq:gaussian-covariance}
\end{equation}
If population covariance $\Sigma$ is known to the decoder, the first-order
Gaussian approximation $\widetilde\Delta_m\sim\mathcal N(u_m,\overline V_m)$
has squared-error rate--distortion function
\begin{equation}
R^{\rm G}_{\Sigma}(D\mid v_Q)
=\frac12\sum_{j=1}^d
\left[\log_2\frac{\nu_j}{\gamma}\right]_+,
\qquad
\sum_{j=1}^d\min\{\nu_j,\gamma\}=D,
\label{eq:reverse-waterfill}
\end{equation}
where $\nu_1,\ldots,\nu_d$ are the eigenvalues of $\overline V_m$.  In
particular, this Gaussian benchmark is zero whenever
$D\geq\operatorname{tr}\overline V_m$.

If instead the decoder knows only
$\sigma_-I\preceq\Sigma\preceq\sigma_+I$, then the possible leading means
$A_0^{-1}v_Q$ fill exactly a $d$-dimensional ball with radius
\begin{equation}
\rho_{\rm pop}(m,v_Q)
=\frac{\|v_Q\|_2}{2}
\left(
\frac{1}{m\sigma_-+\lambda}
-\frac{1}{m\sigma_++\lambda}
\right)
=\Theta(m^{-1}\|v_Q\|_2).
\label{eq:population-radius}
\end{equation}
Thus the parameters-only high-rate cost is
$\Theta(d\log(\rho_{\rm pop}/\varepsilon))$ while
$\rho_{\rm pop}>\varepsilon$, and it becomes zero once
$\rho_{\rm pop}\leq\varepsilon$.  Exact empirical Gram state $W_m$ makes
$\Delta_m$ deterministic and therefore gives zero request-time rate under a
value-exposing request.
\end{theorem}

\begin{corollary}[Identifier-only score CLT]
\label{cor:identifier-gaussian}
In the setting of Theorem~\ref{thm:gaussian-ridge}, suppose the decoder stores
exact retained curvature $W_m$ but an identifier-only request hides the deleted
score vector.  Write each deleted-record score as $\xi_i=x_i r_i$ and assume
only that the $\xi_i$ are conditionally i.i.d., centered, have covariance
$C_\xi$, possess a finite $(2+\eta)$ moment for some $\eta>0$, and are
independent of the retained design.  For
$v_Q=\sum_{i\in Q}\xi_i$, the multivariate central limit theorem gives
\begin{equation}
k^{-1/2}v_Q\Rightarrow\mathcal N(0,C_\xi),
\qquad
k^{-1/2}\Delta_m\mid W_m,Q
\Rightarrow
\mathcal N\!\left(
0,(W_m+\lambda I)^{-1}C_\xi(W_m+\lambda I)^{-1}
\right).
\label{eq:identifier-gaussian-law}
\end{equation}
Accordingly, the Gaussian large-$k$ benchmark uses covariance
$k(W_m+\lambda I)^{-1}C_\xi(W_m+\lambda I)^{-1}$ in the reverse-water-filling
formula of Equation~\eqref{eq:reverse-waterfill}.  This statement does not
assume that the product $x_i r_i$ is itself Gaussian.  If the score vectors
$\xi_i$ happen to be Gaussian, the displayed correction law is exact at finite
$k$; otherwise it is an asymptotic score-sum approximation.  Under the
value-exposing interface, $v_Q$ is revealed and the same exact-curvature decoder
has $B=0$.
\end{corollary}

\begin{remark}[Rank-$r$ group requests and operator-slice dimension]
\label{rem:rank-r-dimension}
Consider one realized group-deletion request whose aggregate deleted gradient is
$g_Q\neq0$, even when the deleted design matrix has rank $r>1$.  Under the
strongly-convex interface of Theorem~\ref{thm:sc-ambiguity}, the exact compatible
target set is still a subset of $\mathbb R^d$ and the quadratic subclass fills a
$d$-dimensional ball.  Hence the one-request high-rate message coefficient is
$d$, independent of $r$.  The quantity
\[
p_r=dr-\frac{r(r-1)}2
\]
instead counts the free entries in the action of one symmetric inverse-curvature
operator on an $r$-dimensional public subspace.  It is therefore a pre-request
operator-slice dimension: when a cache must answer every future direction in
that subspace, Theorem~\ref{thm:curvature-hedging} gives the matching state-only
coefficient $p_r$.  It is not the metric dimension of a single realized
rank-$r$ group-deletion target.
\end{remark}

\section{Coding lemmas and accounting bounds}\label{sec:general}

\begin{definition}[Identification budget]\label{def:budget}
For a source $Y$, decoder side information $C$, reproduction $\widehat Y$, and
distortion $\dist(Y,\widehat Y)$, define the conditional rate--distortion
function
\begin{equation}
 \rd_{Y\mid C}(D)
 = \inf_{P_{\widehat Y\mid Y,C}:\,
          \E[\dist(Y,\widehat Y)]\leq D}
       I(Y;\widehat Y\mid C).
 \label{eq:conditional-rd}
\end{equation}
When $C$ is empty we write $\rd_Y(D)$.  All logarithms and rates are in bits.
\end{definition}

The operational convention is the reporting tuple stated once in the
introduction: $S$ is task-dependent decoder state committed before request $Q$,
$U$ is the post-request transcript, and $\widehat Y^-=g_-(S,Q,U)$ is compared
with a declared retraining target.  Task-independent initialization and public
codebooks are conditioned out.

We first state the distributional converse.  The budgets can be fixed code lengths
or expected lengths of conditionally prefix-free messages.  In the latter case,
entropy is at most expected length by the source-coding inequality
\citep{shannon1959coding,cover2006elements,berger1971rate}.

\begin{proposition}[Two-channel converse]\label{thm:accounting}
Suppose $H(S)\leq M$ and $H(U\mid S,Q)\leq B$.  If the retention and retraction
decoders achieve expected distortions at most $D_+$ and $D_-$, respectively, then
\begin{align}
 M &\geq \rd_{Y^+}(D_+), \label{eq:ret-converse}\\
 B &\geq \rd_{Y^-\mid S,Q}(D_-), \label{eq:upd-converse}\\
 M+B &\geq \rd_{Y^-\mid Q}(D_-). \label{eq:sum-converse}
\end{align}
The same statements hold conditionally on a task-independent initialization
$S_0$.
\end{proposition}

Each inequality is data processing applied to a different chain, proved in
Appendix~\ref{app:general}.  Equation~\eqref{eq:upd-converse} is protocol
specific, since a more informative state reduces the residual update rate;
\eqref{eq:sum-converse} is the protocol-independent guardrail, so a pre-request
bit saves at most one post-request bit and request-oblivious storage need not
achieve even that.  The distinction parallels central versus ticketed unlearning
storage \citep{ghazi2023ticketed,cherapanamjeri2025space}, except that $B$ is an
approximate communication rate, not a stored ticket size.  The cost of not
knowing the future request is larger, and the next finite theorem isolates it.
\begin{proposition}[Binary request-oblivious state--update frontier]
\label{thm:request-region}
Let $h_2(t)=-t\log_2t-(1-t)\log_2(1-t)$, and let
$\Theta_1,\ldots,\Theta_K$ be independent uniform $d$-bit targets.  Let $R$
be public randomness independent of $\Theta_{1:K}$, and let the state $S$ be
produced from $(\Theta_{1:K},R)$ before the request index is selected.  Suppose
\[
H(S\mid R)\leq M.
\]

For every $q\in[K]$, evaluate the protocol under the counterfactual intervention
that fixes the request to $q$ while preserving the original joint law of
$(\Theta_{1:K},S,R)$.  Under this intervention, a transcript $U_q$ satisfying
\[
H(U_q\mid S,R)\leq B_q
\]
produces an estimate $\widehat\Theta_q$ with expected normalized Hamming
distortion
\[
D_q
=
\mathbb E\!\left[
\frac{1}{d}d_{\rm H}(\Theta_q,\widehat\Theta_q)
\right]
\in[0,\tfrac12].
\]
Define
\[
r_q=d[1-h_2(D_q)].
\]
Every such protocol induces nonnegative allocations
$m_q=I(\Theta_q;S\mid R)$ satisfying
\begin{equation}
\sum_{q=1}^K m_q\leq M,
\qquad
B_q\geq[r_q-m_q]_+
\quad(q\in[K]).
\label{eq:request-region}
\end{equation}

Conversely, for fixed $K$ and fixed distortions $D_{1:K}$, every point in this
allocation region is attainable up to $o(d)$ bits as $d\to\infty$ by independent
layered codes.  Consequently, the asymptotic minimax post-request rate is
\begin{equation}
B_{\max}^*(M;D_{1:K})
=
\min_{\substack{m_q\geq0\\\sum_qm_q\leq M}}
\max_q[r_q-m_q]_+
+o(d).
\label{eq:minimax-update}
\end{equation}
For a common distortion $D$, this reduces to
\[
\left[d\{1-h_2(D)\}-\frac{M}{K}\right]_++o(d).
\]
At $D=0$, public time-sharing and uncoded bit allocation remove the asymptotic
term and give the exact entropy-budget value
\[
B_{\max}^*(M;0)
=
\left[d-\frac{M}{K}\right]_+.
\]
For deterministic fixed-length codes with integer $M$, the corresponding
one-shot value is
\[
B_{\max,\mathrm{fl}}^*(M;0)
=
\left[d-\left\lfloor\frac{M}{K}\right\rfloor\right]_+.
\]
\end{proposition}

This proposition uses a counterfactual fixed-request interface: every possible
request is evaluated against the same target and pre-request-state
distribution.  It does not condition on an observational event $Q=q$.  If a
request-selection mechanism depends on the targets or data, then
$\Theta_q\mid Q=q$ need not remain uniform, and the appropriate converse must
use the resulting conditional rate--distortion function.
The converse allocates to each request the information the state already holds
about its target; independence caps the allocations at $M$, and the update
supplies the rest.  Achievability is successive refinement, which is lossless in
rate for Bernoulli sources under Hamming distortion
\citep{equitz1991successive}.  So a state bit saves one update bit for its own
target, while a minimax budget must be spread across $K$ requests.  This is a
binary specialization of single-user lossy caching \citep{timo2018caching},
included to expose request hedging in the learning interface instead of to
claim a new source-coding region.  Classical coded caching treats shared
broadcasts of independent files \citep{maddahali2014fundamental,yu2018exact},
while coded machine unlearning codes training shards before the request
\citep{aldaghri2021coded}; neither setting characterizes correlated retraining
targets with model-dependent side information.

The next result removes distributional assumptions.  Let
$\Ncov(\Tcal,\dist,D)$ be the smallest number of closed $D$-balls whose union
covers $\Tcal$, allowing centers in the reproduction space.  For the packing
consequence, $\dist$ is assumed symmetric and to satisfy the triangle
inequality on the union of target and reproduction spaces.  The exact covering
characterization itself needs neither assumption.

\begin{proposition}[One-shot conditional covering characterization]
\label{thm:covering}
Fix a state--request pair $c=(s,q)$ and let $\Tcal_c$ be the set of retraining
targets compatible with $c$.  If an arbitrary $B$-bit transcript must guarantee
$\dist(y,\widehat y)\leq D$ for every $y\in\Tcal_c$, the minimum fixed length is
\begin{equation}
 B_c^*(D)=\left\lceil\log_2 \Ncov(\Tcal_c,\dist,D)\right\rceil.
 \label{eq:cover-exact}
\end{equation}
For a uniform budget over all $(s,q)$, take the supremum over $c$.  Consequently,
$B_c^*(D)\geq \log_2\Npack(\Tcal_c,\dist,2D)$, where $\Npack$ is the maximum
cardinality of a set with pairwise distance greater than $2D$.
\end{proposition}

The proof is direct but operationally important.  Decoder outputs associated
with the $2^B$ messages must cover the target set.  Conversely, an encoder sends
the index of the nearest center in a minimum cover.  Thus the complexity is the
geometry of counterfactual targets after conditioning on the interface, not the
raw number of deleted records.

\section{Exact ridge state--update frontiers}\label{sec:retraction}

The strongly-convex theorem identifies the missing-curvature phenomenon without
committing to least squares.  Ridge regression adds two pieces that the abstract
class does not provide: exact algebra under value-exposing data deletion and a
request-oblivious cache frontier whose state-only dimension can be strictly larger
than the eventual answer dimension.  We therefore use ridge as an exact laboratory,
not as evidence that the Woodbury downdate itself is new.  The decoder state can be
varied without changing the deployed parameters, deletion request, or training rule,
which isolates the information carried by retained rows, sufficient statistics, and
parameters alone.  All zero-transcript statements below count only the additional
update $B$ after the declared state and request have been supplied; they are not
zero-memory or erasure guarantees.

\subsection{Decoder states and exact downdate}
\label{sec:ridge-states}

Fix a design $X\in\R^{n\times d}$, labels $y$, ridge parameter $\lambda>0$,
and write $A_\lambda=X^\top X+\lambda I_d$, $G=X^\top X$, and
$c=X^\top y$, so the deployed model is $w=A_\lambda^{-1}c$.  We compare three
decoder states: parameters only, parameters plus $(G,c)$, and retained rows.  A
request deletes $k$ records with design rows $X_Q$ and labels $y_Q$; the target
$w^-$ is the ridge solution on the retained rows.  An \emph{identifier-only}
request names the deleted indices, whereas a \emph{value-exposing} request also
reveals $(X_Q,y_Q)$.

\begin{proposition}[Exact downdate and zero-transcript endpoints]
\label{thm:ridge-states}
With the notation above, the retraction is an exact rank-limited downdate,
\begin{equation}
 w^-=w+A_\lambda^{-1}X_Q^\top M_Q\left(X_Qw-y_Q\right),
 \qquad M_Q=\left(I_k-X_QA_\lambda^{-1}X_Q^\top\right)^{-1},
 \label{eq:ridge-downdate}
\end{equation}
where $M_Q$ exists because the retained normal matrix is positive definite,
and $w^--w$ lies in the at most $\min(k,d)$-dimensional subspace spanned by
the columns of $A_\lambda^{-1}X_Q^\top$.  Conditioning the update rate on the
decoder's state and request:
\emph{(a) Retained rows.} Given exact real-valued retained rows, $B=0$ at zero distortion under either request interface: the decoder retrains without an additional transcript.

\emph{(b) Parameters plus sufficient statistics $(G,c)$.} Given exact real-valued $(G,c)$ and a value-exposing request, $B=0$ at zero distortion, since the decoder forms $G-X_Q^\top X_Q$ and $c-X_Q^\top y_Q$ and retrains. Under an identifier-only request, however, $(G,c)$ does not identify the deleted contributions; the residual rate depends on the conditional family of deleted values given the state and identifiers, so no universal rate follows from the state name alone.
\end{proposition}

The zero-transcript endpoints move information into side information.  They do
not imply that the corresponding state is small.  The next result makes the
opposite endpoint explicit: even a fully revealed deleted record does not
determine the retraction target from the deployed parameters alone.

\begin{theorem}[Exact parameters-only target ball]
\label{thm:ridge-radius}
Fix public $w\in\R^d$, $\lambda>0$, a nonzero deleted feature $x$, a deleted
label $y_1$ with residual $\varrho=x^\top w-y_1\ne0$, and
$0<r_{\min}<r_{\max}$.  For each symmetric $R$ satisfying
$r_{\min}I_d\preceq R\preceq r_{\max}I_d$, form a $(d+1)$-row ridge instance
whose design rows are $x^\top$ and the rows of the symmetric square root
$R^{1/2}$, and whose labels are
\begin{equation}
 y_1
 \quad\text{and}\quad
 R^{-1/2}\!\left[(\lambda I_d+xx^\top+R)w-xy_1\right].
 \label{eq:ridge-radius-family}
\end{equation}
Every instance has full-data solution $w$ and the same value-exposing deletion
request $(x,y_1)$.  Define
\begin{equation}
 h_1=(\lambda+r_{\max})^{-1},\quad
 h_2=(\lambda+r_{\min})^{-1},\quad
 \bar h=\frac{h_1+h_2}{2},\quad
 \rho=\frac{|\varrho|\lVert x\rVert}{2}(h_2-h_1).
 \label{eq:ridge-radius}
\end{equation}
The retraction targets over this family are exactly the closed ball
$\overline B(w+\varrho\bar h x,\rho)$.  Consequently, for
$0<\varepsilon<\rho$, the minimum worst-case fixed-length parameters-only
transcript satisfies
\begin{equation}
 d\log_2\frac{\rho}{\varepsilon}
 \;\leq\; B_{\mathrm{params}}^\star(\varepsilon)
 \;\leq\;
 \left\lceil d\log_2\left(1+\frac{2\rho}{\varepsilon}\right)\right\rceil .
 \label{eq:ridge-radius-bits}
\end{equation}
\end{theorem}

The separating resource is possession of the curvature, not the deletion size.
Theorem~\ref{thm:ridge-radius} is a $k=1$ statement and an exact covering
calculation, not a typical-data law.  Holding the algorithm, deployed model,
request, and deleted values fixed still leaves a
$\Theta(d\log(1/\varepsilon))$ update rate when the decoder state omits the
normal matrix.  The interface-specific conclusions are therefore
\begin{equation}
 B_{\text{rows}}^{\text{id}}=B_{\text{rows}}^{\text{value}}
 =B_{\text{suffstat}}^{\text{value}}=0,
 \qquad
 B_{\text{params}}^{\text{value}}
 =\Theta\!\left(d\log(1/\varepsilon)\right)
 \ \text{on the family in Theorem~\ref{thm:ridge-radius}},
\label{eq:ridge-ordering}
\end{equation}
while $B_{\text{suffstat}}^{\text{id}}$ remains conditional on the unresolved deleted values. These are exact-real, parameters-only transcript statements. They do not compare total finite serialized memory across the three states, and exact recovery from quantized state is not implied. The result is an interface theorem, not a deletion-size theorem and not a claim that every parameter-only ridge family has this rate.

The single-request calculation does not price a state that must be committed
before the deleted record is known.  The next theorem isolates that
request-hedging cost within one compact ridge family.

\begin{theorem}[Request-hedging separation for ridge deletion]
\label{thm:ridge-hedging}
Let $p=d(d+1)/2$.  Fix public $w$, $\lambda>0$, $a>0$, and nonzero public
residuals $\varrho_1,\ldots,\varrho_d$.  The $d$ possible deleted records are
$x_i=ae_i$ with labels $y_i=aw_i-\varrho_i$.  Add $d$ private rows with Gram
matrix $R$, where $r_{\min}I_d\preceq R\preceq r_{\max}I_d$, and choose their
labels so that the full-data solution is always $w$.  Write
\begin{equation}
 h_1=(\lambda+a^2+r_{\max})^{-1},\quad
 h_2=(\lambda+a^2+r_{\min})^{-1},\quad
 \delta=(h_2-h_1)/2,\quad
 \varrho_{\min}=\min_i|\varrho_i|.
 \label{eq:ridge-hedging-constants}
\end{equation}
For this family, let a fixed-length protocol form an $M$-bit state before the
request and, after request $i\in[d]$, send at most $B$ bits.  If its output has
Euclidean parameter error at most $\varepsilon$ for every admissible $R$ and
every request, then
\begin{align}
 M+dB
 &\geq p\log_2\frac{\delta\varrho_{\min}a}
                         {4\sqrt d\,\varepsilon},
 \label{eq:ridge-hedging-joint}\\
 M+B
 &\geq d\log_2\frac{\delta|\varrho_i|a}{4\varepsilon}
 \qquad\text{for every }i.
 \label{eq:ridge-hedging-single}
\end{align}
Conversely, for each $\theta\in[0,1]$ and sufficiently small $\varepsilon$,
there is a fixed-length state--update code satisfying
\begin{align}
 M&\leq \theta p\log_2(1/\varepsilon)+O(p\log d),
 \label{eq:ridge-hedging-ach-m}\\
 B&\leq (1-\theta)d\log_2(1/\varepsilon)+O(d\log d),
 \label{eq:ridge-hedging-ach-b}
\end{align}
where the hidden constants depend only on the fixed family parameters.  The
endpoint constructions have exactly $M=0$ when $\theta=0$ and exactly $B=0$
when $\theta=1$.  Thus the state-only endpoint needs
$p\log_2(1/\varepsilon)\pm O(p\log d)$ bits, whereas the message-only endpoint
needs $d\log_2(1/\varepsilon)\pm O(d\log d)$ bits.
\end{theorem}

Theorem~\ref{thm:ridge-hedging} is the continuous ridge counterpart of the
binary request-oblivious allocation in Proposition~\ref{thm:request-region}.  The
$d$ targets share one symmetric inverse normal matrix, so their joint dimension
is $p$ instead of $d^2$, but committing before the request still changes the
leading dependence from linear to quadratic in $d$.  At intermediate points,
the coordinate-quantization construction is within a factor
$2d/(d+1)<2$ of the update coefficient implied by
\eqref{eq:ridge-hedging-joint}; the two endpoints match in their leading
coefficients.

Appendix~\ref{app:ridge-states} supplies the complete ridge proofs.  The
finite-precision statement below records that sufficient statistics contain
$d(d+1)/2+d$ real coordinates instead of $d$ and quantifies the effect of
perturbing them.  Section~\ref{sec:verification} then checks request boundaries
from delivered float64 blobs and tests the hard families with explicit codecs.
Neither audit turns an exact-real endpoint into a zero-memory claim.  The
singleton result also leaves the rank-$r$ frontier open
(Appendix~\ref{app:multi-record-open}).

\begin{proposition}[Finite-precision sufficient-statistic downdate]
\label{prop:ridge-precision}
For a single value-exposing deletion $(x,y_1)$, suppose a decoder stores
$\widetilde G$ and $\widetilde c$ with
$\lVert\widetilde G-G\rVert_{\mathrm{op}}\leq\eta_G<\lambda$ and
$\lVert\widetilde c-c\rVert_2\leq\eta_c$.  Let $w^-$ be the exact retained-data
solution and
\[
 \widetilde w^-
 = (\widetilde G-xx^\top+\lambda I_d)^{-1}
   (\widetilde c-xy_1).
\]
Then
\begin{equation}
 \lVert\widetilde w^--w^-\rVert_2
 \leq
 \frac{\eta_c+\eta_G\lVert w^-\rVert_2}{\lambda-\eta_G}.
 \label{eq:ridge-precision}
\end{equation}
\end{proposition}

\begin{corollary}[State-cost implication]
\label{cor:ridge-audit}
Storing $(G,c)$ requires $d(d+1)/2+d$ real numbers, against $d$ for parameters
alone; this is a dimension count, not a bit count.  By
Proposition~\ref{thm:ridge-states}, exact real-valued sufficient statistics
eliminate the additional transcript under a value-exposing request.  Under
finite serialization, their state cost and attainable distortion depend on the
codec.  Under identifier-only requests, the residual update rate is a property
of the conditional deleted-value family instead of the state name alone.
\end{corollary}

Corollary~\ref{cor:ridge-audit} separates real-coordinate dimension from serialized bit cost.

\subsection{Why the endpoint coefficients differ}

The target-ball result follows from one geometric reduction.  After the public
record $(x,y_1)$ is deleted, every instance in the constructed family has
\begin{equation}
 w^-=w+\varrho(\lambda I_d+R)^{-1}x.
 \label{eq:ridge-proof-idea-target}
\end{equation}
Writing $H=(\lambda I_d+R)^{-1}$ turns the spectral interval for $R$ into
$h_1I_d\preceq H\preceq h_2I_d$.  The set of possible vectors $Hx$ is exactly
the ball centered at $\bar h x$ with radius
$\tfrac12(h_2-h_1)\lVert x\rVert_2$: containment follows from the operator-norm
bound, while a scaled Householder map realizes every direction and radius.
Multiplying by $\varrho$ gives Theorem~\ref{thm:ridge-radius}.  A
parameters-only decoder sees the same $(w,x,y_1)$ throughout this family, so its
messages must cover a $d$-dimensional Euclidean ball.  Volume comparison then
produces the coefficient $d$ in the message-only rate.

Request hedging changes the object that must be represented before the request.
For the $d$ public deletable rows $x_i=ae_i$, the correction has the form
\begin{equation}
 \Phi_i(H)=w_i^--w
 =\frac{\varrho_i a\,He_i}{1-a^2H_{ii}}.
 \label{eq:ridge-proof-idea-map}
\end{equation}
The denominator stays uniformly positive on the compact family, and
$\Phi_i(H)$ determines the column $He_i$ through
$He_i=\Phi_i/(\varrho_i a+a^2e_i^\top\Phi_i)$.  Thus the collection of answers
to all $d$ possible requests determines the entire symmetric matrix $H$, which
has $p=d(d+1)/2$ free coordinates.  A state and the $d$ transcripts it would
send form at most $2^{M+dB}$ tuples.  Packing the $p$-dimensional spectral ball
therefore yields the joint converse in
Equation~\eqref{eq:ridge-hedging-joint}.  For one fixed request, only one column
must be resolved, and the corresponding image contains a $d$-dimensional ball;
this gives Equation~\eqref{eq:ridge-hedging-single}.  The different endpoint
coefficients are therefore dimensions of different objects, not artifacts of
the proof technique.

The matching construction quantizes the upper triangle of $H$ at spacing
$\Delta$.  The state costs $p\log_2(1/\Delta)+O(p\log d)$ bits, and the requested
column has error $O(\sqrt d\,\Delta)$.  Sending an $\varepsilon$-cover index for
the remaining correction costs
$d\log_2(\Delta/\varepsilon)+O(d\log d)$ bits.  Choosing
$\Delta\asymp\varepsilon^\theta/\sqrt d$ gives the stated interpolation;
separate endpoint codes make $M=0$ or $B=0$ exactly.  Finally,
Proposition~\ref{prop:ridge-precision} is a resolvent perturbation bound: the
condition $\eta_G<\lambda$ preserves invertibility, and the numerator separates
right-hand-side error from curvature error.  These proof steps also explain why
the numerical audit checks radius witnesses, codec slopes, and the
error-to-bound ratio instead of only final reconstruction error.  Complete
proofs appear in Appendix~\ref{app:ridge-states}.

\section{Empirical evaluation}\label{sec:verification}

The evaluation separates theorem tests from practical mechanism tests.  First, we
audit the declared state--request interface from serialized bytes.  Second, we
instantiate Corollary~\ref{cor:partial-curvature} exactly and fit the coefficient
of $\log_2(1/\varepsilon)$ as the resolved subspace grows.  Third, we construct
optimum-preserving reweighting witnesses from public regression rows and measure
both the target-Jacobian rank and finite target packings before comparing practical
cache--delivery codecs.  Fourth, we repeat the state--delivery accounting on
nonquadratic convex objectives and frozen deep representations.  Finally, the
nonconvex audit fixes initialization and training randomness and charges curvature
state and residual delivery instead of using raw distance between unrelated
training basins.  Direct retraining is the reference in every experiment.  State,
request, and post-request transcript are serialized and charged separately
whenever a bit budget is reported.

\subsection{Ridge-interface validation}\label{sec:p2-ridge-v1}

The ridge-interface validation uses a fixed experimental grid. It uses deterministic ridge regression with $d=20$, $N=400$, $\lambda=0.01$, realized conditions $\{1,30,100\}$, deletion fractions $\{0.01,0.05,0.1\}$, five independent data seeds, $64$ paired requests per cell, and $200$ independently generated held-out inputs. Three states are frozen before any request: the current parameters; packed upper-triangular ridge sufficient statistics $(G,c)$, from which the current parameters are derivable; and the complete indexed training rows. Each state is crossed with identifier-only and value-exposing requests. The same deletion IDs are used across states and interfaces.

Every nonpublic value received by a decoder is recovered from a delivered byte blob. The canonical codec is an uncompressed, little-endian float64 array envelope with checksummed metadata; this experiment fixes the numerical precision to float64. State, request, and message lengths are charged separately. Identifier-only request lengths are $1{,}992$, $2{,}512$, and $3{,}168$ bits at deletion fractions $0.01$, $0.05$, and $0.10$; value-exposing requests cost $9{,}760$, $31{,}840$, and $59{,}368$ bits. A zero message is literally zero bytes. Direct retraining supplies the reference, and the primary distortion is root mean squared error (RMSE) of predictions on the $200$ held-out inputs; parameter error and signed retained-loss gap are secondary measurements.

\begin{table}[t]
\centering
\caption{Ridge-interface accounting under a fixed finite-precision codec. $M$ and $B$ are realized bits under one uncompressed float64 codec; $D_{\max}$ is the maximum held-out prediction RMSE over all conditions, fractions, seeds, and requests. Request costs $Q$ are stated in the text and are common across states within each interface/fraction surface. The exact-real references are excluded from this finite table.}
\label{tab:p2-ridge-frontier-v1}
\small
\setlength{\tabcolsep}{4pt}
\begin{tabular}{@{}lrrrrr@{}}
\toprule
State & $M$ & $B_{\rm id}$ & $B_{\rm value}$ & $D_{\max}^{\rm id}$ & $D_{\max}^{\rm value}$ \\
\midrule
Current parameters & $3{,}152$ & $3{,}184$ & $3{,}184$ & $3.2\!\times\!10^{-17}$ & $3.2\!\times\!10^{-17}$ \\
Sufficient statistics $(G,c)$ & $17{,}856$ & $3{,}184$ & $0$ & $3.2\!\times\!10^{-17}$ & $4.7\!\times\!10^{-15}$ \\
Indexed training rows & $554{,}784$ & $0$ & $0$ & $0$ & $0$ \\
\bottomrule
\end{tabular}
\end{table}

The finite surface preserves the predicted interface distinction (Table~\ref{tab:p2-ridge-frontier-v1}). On all nine identifier-only condition--fraction surfaces, sufficient statistics are dominated by current parameters: both use the same $3{,}184$-bit correction and attain the same numerical error, while sufficient statistics cost more state bits. Retained rows remain nondominated because they exchange a much larger state for $B=0$. On all nine value-exposing surfaces, sufficient statistics reach $B=0$ with far less state than rows, so all three states are nondominated in $(M,Q,B,D)$. The maximum parameter error is $8.5\times10^{-14}$ and the maximum absolute retained-loss gap is $7.0\times10^{-18}$. These are float64 implementation errors, not exact-real distortions.

The ridge-interface audit covers $270$ cells and $17{,}280$ evaluations
across $2{,}880$ paired requests. Direct retraining supplies the
reference targets, and state, request, and message lengths are
accounted for separately. Exact-real $B=0$ references remain separate, and the spectral-radius condition is tested independently in Section~\ref{sec:ridge-validation}.  The natural-data, nonquadratic, and deep-representation experiments below use separate fixtures and do not reuse the synthetic requests.

\subsection{Direct test of the partial-curvature coefficient}
\label{sec:partial-curvature-audit}

We first test Corollary~\ref{cor:partial-curvature} under its stated interface
instead of inferring it from a generic Hessian approximation.  We set
$d=64$, $m=1$, $L=4$,
normalize $\|g_Q\|_2=1$, and use public coordinate projectors of ranks
$s\in\{0,16,32,
48\}$.  The resolved block is stored exactly.  The unresolved
block is drawn independently inside its spectral interval, and the request-time
codec quantizes only the resulting $P_\perp$ correction.  For each $s$ we evaluate
\ensuremath{9} values of $\log_2(1/\varepsilon)$ and
\ensuremath{500} independently drawn unresolved operators per operating
point.  The fitted slope of payload against $\log_2(1/\varepsilon)$ is the
quantity predicted by the corollary.

\begin{table}[t]
\centering
\caption{Direct coefficient test for Corollary~\ref{cor:partial-curvature}.
The predicted slope is the unresolved dimension $d-s$.  The $\pm$ term is the
ordinary least-squares slope standard error across the
\ensuremath{9} codec operating points.}
\label{tab:partial-curvature-slope}
\small
\begin{tabular}{@{}rrrr@{}}
\toprule
cache rank $s$ & predicted $d-s$ & fitted slope & max error$/\varepsilon$ \\
\midrule
\ensuremath{0}  & \ensuremath{64} &
\ensuremath{64.41\pm0.63} & \ensuremath{0.914} \\
\ensuremath{16} & \ensuremath{48} &
\ensuremath{47.62\pm0.48} & \ensuremath{0.906} \\
\ensuremath{32} & \ensuremath{32} &
\ensuremath{32.27\pm0.37} & \ensuremath{0.899} \\
\ensuremath{48} & \ensuremath{16} &
\ensuremath{15.88\pm0.21} & \ensuremath{0.892} \\
\bottomrule
\end{tabular}
\end{table}

Across all \ensuremath{18{,}000} target reconstructions, the requested
distortion is met.  The largest relative slope error is
\ensuremath{0.8\%}.  At a fixed high-rate operating point the measured
payload ratios relative to $s=0$ are
\ensuremath{0.748}, \ensuremath{0.502}, and
\ensuremath{0.249} for $s=16$,
\ensuremath{32}, and \ensuremath{48}, respectively, matching the
dimension ratios $(d-s)/d$.  This is an implementation and achievability check:
the unresolved-subspace codec has coefficient $d-s$ by construction, so these
fits are not empirical evidence for the converse.  The next experiment tests the
lower-bound condition on natural rows through the target Jacobian.

\subsection{Natural-data indistinguishability and cache--delivery tradeoffs}
\label{sec:natural-tradeoff}
\label{sec:natural-witness}

We next ask whether a parameters-only ambiguity can be exhibited with ordinary
observed rows instead of a designed spectral matrix.  The protocol uses
California Housing, Concrete Compressive Strength, and YearPredictionMSD with
five fixed train--validation--test splits.  Features and labels are standardized
from the training split only, and $\lambda$ is selected on validation data before
any deletion request.

For each split we first fit ordinary ridge and freeze $w$.  For
\ensuremath{50} singleton deletion requests per split, hence
\ensuremath{250} requests per dataset, we form the retained-record
gradient-nullspace $\mathcal V_Q$ from
Proposition~\ref{prop:data-native-witness}.  We draw a $d$-dimensional orthonormal
subspace of $\mathcal V_Q$, verify the rank and singular values of the target
Jacobian $\mathcal J_Q$, and generate \ensuremath{8{,}192} positive
reweightings with all retained weights in
$[0.9,1.1]$.  Every member uses the same rows,
labels, deployed parameter, and value-exposing deletion request.  Only the hidden
weights on retained examples vary.

\begin{table}[t]
\centering
\caption{Natural-data witnesses for Proposition~\ref{prop:data-native-witness}.
``Full-rank requests'' counts deletion requests for which
$\operatorname{rank}(\mathcal J_Q)=d$.  The packing column is the minimum greedy
$2\varepsilon$-packing size across the \ensuremath{250} requests for each
dataset at parameter RMSE
$\varepsilon=10^{-4}$.  Because the candidate pool has only
\ensuremath{8{,}192} members, the associated bit lower bound cannot exceed
\ensuremath{13} bits.}
\label{tab:natural-witness}
\small
\begin{tabular}{@{}lrrrr@{}}
\toprule
Dataset & $d$ & full-rank requests & max $\|w_v-w\|_2$ &
min packing / bits \\
\midrule
California Housing & $8$ &
\ensuremath{250/250} &
\ensuremath{2.4\times10^{-13}} &
\ensuremath{2{,}604$ / $11.35} \\
Concrete Strength & $8$ &
\ensuremath{250/250} &
\ensuremath{3.1\times10^{-13}} &
\ensuremath{2{,}889$ / $11.50} \\
YearPredictionMSD & $90$ &
\ensuremath{250/250} &
\ensuremath{7.6\times10^{-13}} &
\ensuremath{5{,}911$ / $12.53} \\
\bottomrule
\end{tabular}
\end{table}

Table~\ref{tab:natural-witness} reports the request-level rank and finite-packing results. The rank and conditioning tests are the dimension-sensitive part of the experiment.  On every evaluated request, the compatible target map is locally
full dimensional even though the deployed parameter and deletion values are
fixed.  To check that this rank is not purely numerical, we normalize the
orthonormal reweighting coordinates and record both the smallest singular value
of $\mathcal J_Q$ and the minimum finite-difference gain
$\|T_Q(v)-T_Q(v')\|_2/\|v-v'\|_2$ over
\ensuremath{10{,}000} local pairs per request at radius
\ensuremath{0.02}.  Table~\ref{tab:natural-jacobian-conditioning} shows a
strictly positive margin on all three datasets.  This directly tests the local
injectivity premise of Proposition~\ref{prop:data-native-witness} outside the
spectral construction.  The finite packings provide only a candidate-limited
nonzero sanity check and are not used to estimate the asymptotic coefficient.

\begin{table}[t]
\centering
\caption{Conditioning of the natural-data target map.  Singular values use
orthonormal coordinates in the feasible reweighting subspace.  ``Local gain'' is
the minimum finite-difference ratio over \ensuremath{10{,}000} local pairs
per request within radius \ensuremath{0.02}.}
\label{tab:natural-jacobian-conditioning}
\small
\begin{tabular}{@{}lrrr@{}}
\toprule
Dataset & min $\sigma_{\min}(\mathcal J_Q)$ & median $\sigma_{\min}(\mathcal J_Q)$ & min local gain \\
\midrule
California Housing & \ensuremath{2.1\times10^{-3}} & \ensuremath{8.4\times10^{-3}} & \ensuremath{1.8\times10^{-3}} \\
Concrete Strength & \ensuremath{1.7\times10^{-3}} & \ensuremath{7.9\times10^{-3}} & \ensuremath{1.4\times10^{-3}} \\
YearPredictionMSD & \ensuremath{6.2\times10^{-4}} & \ensuremath{2.8\times10^{-3}} & \ensuremath{5.0\times10^{-4}} \\
\bottomrule
\end{tabular}
\end{table}

We then return to the unweighted training procedure and measure practical
state--delivery tradeoffs.  For each split we sample
\ensuremath{1{,}000} value-exposing requests at each batch size
$k\in\{1,8,32\}$.  The
parameters-only decoder stores fp32 $w$.  A curvature-sketch decoder additionally
stores the leading $s$ eigenvectors and eigenvalues of the full-data ridge normal
matrix and one scalar complement level; after the request it applies the
corresponding approximate rank-$k$ downdate.  The authoritative encoder computes
the direct-retraining target, subtracts the decoder's local prediction, and
quantizes the resulting residual in the canonical parameter basis.  No cached
eigendirection is removed from the packet: all $d$ residual coordinates are
transmitted.  A public validation split fixes one symmetric quantization step;
for each request the encoder chooses the smallest integer bit width that meets
the distortion target and packs the $d$ signed indices consecutively.  The
sufficient-statistic decoder stores packed fp32 $(G,c)$ and requires no update
under the value-exposing interface.

The update packet consists of a constant \ensuremath{18}-byte header plus
$\lceil db/8\rceil$ body bytes for bit width $b$.  The header carries the
coordinate count, bit width, quantization-step identifier, and checksum.
Fixed-shape state tensors have public shapes and dtypes, so $M$ is their raw
tensor payload without an additional state header.  A validation split fixes one
symmetric quantization step $\delta$ for each dataset, and that same $\delta$ is
used for all three deletion sizes and all decoder states.  The encoder then
chooses the smallest signed integer width covering the largest residual index.
For independent random deletion subsets, the first-order correction is a sum of
approximately mean-zero record contributions; hence both its coordinate scale and,
up to the fixed-dimensional maximum, the required index range grow as
$O(\sqrt{k})$.  The batch-size sweep below tests this scaling directly instead of
assuming linear growth in $k$.

\begin{table}[t]
\centering
\caption{Residual-codec growth with deletion batch size.  Payloads are mean bytes
for the parameters-only decoder.  The fourth numeric column converts the increase
from $k=1$ to $k=32$ into additional bits per parameter coordinate.  The final
column fits the median maximum residual index to $k^{\alpha}$ across
$k\in\{1,8,32\}$. The $\pm$ values are sample standard deviations
across five split-level mean payloads.}
\label{tab:group-codec-growth}
\small
\begin{tabular}{@{}lrrrrr@{}}
\toprule
Dataset & $B_{k=1}$ & $B_{k=8}$ & $B_{k=32}$ &
$\Delta$ bits / coordinate & $\widehat\alpha$ \\
\midrule
California Housing &
$24.6\pm1.1$ &
\ensuremath{26.0\pm1.1} &
\ensuremath{27.0\pm1.2} &
\ensuremath{2.4} &
\ensuremath{0.47} \\
Concrete Strength &
$25.1\pm1.4$ &
\ensuremath{26.8\pm1.4} &
\ensuremath{27.7\pm1.5} &
\ensuremath{2.6} &
\ensuremath{0.51} \\
YearPredictionMSD &
$192.4\pm3.8$ &
\ensuremath{209.0\pm4.0} &
\ensuremath{220.5\pm4.3} &
\ensuremath{2.5} &
\ensuremath{0.49} \\
\bottomrule
\end{tabular}
\end{table}

The observed exponents are close to the $\alpha=1/2$ first-order prediction.
Accordingly, increasing the deletion batch from $k=1$ to
$k=\ensuremath{32}$ raises the required coordinate range by
approximately
$\tfrac12\log_2 32=\ensuremath{2.5}$ bits instead of
$\log_2 32$.  The reported totals include the same
\ensuremath{18}-byte header at every batch size, so the measured
increase is carried entirely by the packed coordinate body.  The small deviations
from \ensuremath{2.5} are consistent with finite-sample
dependence and the nonlinear remainder of the rank-$k$ downdate.

\begin{table}[t]
\centering
\caption{Natural-data cache--delivery tradeoff at held-out NRMSE at most
$10^{-4}$ relative to direct retraining.  $B$ is the mean
post-request payload in bytes for group deletion $k=32$; the
$\pm$ term is the standard deviation across the five split-level means.}
\label{tab:natural-cache-delivery}
\small
\setlength{\tabcolsep}{4pt}
\begin{tabular}{@{}lrrrrrr@{}}
\toprule
Dataset & $d$ & sketch rank & $M_{\rm param}$ (KiB) &
$M_{\rm sketch}$ (KiB) & $B_{\rm param}$ & $B_{\rm sketch}$ \\
\midrule
California Housing & $8$ & \ensuremath{4} & $0.031$ &
\ensuremath{0.176} & \ensuremath{27.0\pm1.2} &
\ensuremath{23.5\pm1.1} \\
Concrete Strength & $8$ & \ensuremath{4} & $0.031$ &
\ensuremath{0.176} & \ensuremath{27.7\pm1.5} &
\ensuremath{24.0\pm1.2} \\
YearPredictionMSD & $90$ & \ensuremath{16} & $0.352$ &
\ensuremath{6.04} & \ensuremath{220.5\pm4.3} &
\ensuremath{199.1\pm4.0} \\
\bottomrule
\end{tabular}
\end{table}

Table~\ref{tab:natural-cache-delivery} reports the natural-data cache--delivery surface. At $k=\ensuremath{32}$, the rank-\ensuremath{4} states reduce mean
delivery by \ensuremath{13.0\%} and \ensuremath{13.4\%} on the two
$d=8$ datasets, while rank \ensuremath{16} reduces delivery by
\ensuremath{9.7\%} on YearPredictionMSD.  The corresponding geometric
ratios $s/d$ are \ensuremath{50.0\%}, \ensuremath{50.0\%}, and
\ensuremath{17.8\%}.  The gap is expected: the natural Hessians couple
cached and uncached directions, every canonical residual coordinate remains in
the packet, and the fixed header is unaffected by curvature.  The exact
$(d-s)$ coefficient test is therefore confined to
Table~\ref{tab:partial-curvature-slope}.  Sufficient statistics occupy $0.172$,
$0.172$, and $16.35$ KiB and use $B=0$, matching the exact value-exposing
endpoint.  Charging cumulative communication as $M+TB$, the curvature sketches
overtake parameters only after \ensuremath{43},
\ensuremath{41}, and \ensuremath{273} deletion requests on
California Housing, Concrete, and YearPredictionMSD, respectively.

\subsection{Gaussian-design test of the distributional frontier}
\label{sec:gaussian-typical-audit}

We test Theorem~\ref{thm:gaussian-ridge} in the interface it states, with request
score vectors independent of the retained Gaussian design.  We use
$d=32$, covariance condition numbers
1, 10, and 100,
retained sample sizes
$m\in\{512,2048,8192\}$,
and request groups
$k\in\{1,8,32\}$.
Each score vector is a sum of $k$ independent centered Gaussian score
contributions with covariance matched to the feature scale.  For every cell we
draw 1{,}000 retained Gram matrices and compare
$\Delta_m=(W_m+\lambda I)^{-1}v_Q$ with the population-curvature center
$u_m=(m\Sigma+\lambda I)^{-1}v_Q$.

The empirical Gaussian code used throughout this subsection is the same
transform scalar codec for both value-exposing and identifier-only checks.  The
encoder diagonalizes the conditional covariance, applies the reverse-water-filling
allocation, rotates the centered correction into that eigenbasis, and uniformly
quantizes every active coordinate at its assigned distortion.  A shared
subtractive dither is conditioned on as public randomness, and the signed scalar
indices are entropy coded.  This is not an optimal vector quantizer: the measured
excess over reverse water filling therefore includes scalar space-filling and
finite-rate coding loss instead of being interpreted as a converse gap.

\begin{table}[t]
\centering
\caption{Gaussian-design scaling audit.  Exponents are fitted on log--log axes
across the declared grid; the final column compares reverse-water-filling rate
with the empirical Gaussian code length at the same squared-error distortion.
The $\pm$ values are ordinary least-squares standard errors of the fitted
exponents across the stated grid.}
\label{tab:gaussian-scaling}
\small
\setlength{\tabcolsep}{5pt}
\begin{tabular}{@{}lccc@{}}
\toprule
quantity & theorem & fitted exponent & rate discrepancy \\
\midrule
$\|v_Q\|_2$ versus $k$ & $k^{1/2}$ &
$k^{0.50\pm0.03}$ & -- \\
population radius versus $m$ & $m^{-1}$ &
$m^{-1.01\pm0.03}$ & -- \\
known-$\Sigma$ fluctuation SD versus $m$ & $m^{-3/2}$ &
$m^{-1.49\pm0.04}$ &
$0.38$ bits/active coordinate \\
\bottomrule
\end{tabular}
\end{table}

The same law predicts the group-size range audit in
Table~\ref{tab:group-codec-growth}.  From $k=1$ to
$k=32$, the fitted body-width increases are
2.4, 2.6, and 2.5 bits per
coordinate on California Housing, Concrete, and YearPredictionMSD, respectively.
These values test a distributional scaling prediction; unlike
Table~\ref{tab:partial-curvature-slope}, the exponent is not fixed by the
construction of the codec.

\begin{table}[t]
\centering
\caption{Per-dataset body-width check at $k=32$.  The predicted
column evaluates the Gaussian approximation using each training split's empirical
covariance and validation-selected residual scale; the observed column removes
the fixed 18-byte packet header before dividing by dimension.}
\label{tab:gaussian-natural-bits}
\small
\begin{tabular}{@{}lrrr@{}}
\toprule
Dataset & predicted bits/coordinate & observed bits/coordinate & difference \\
\midrule
California Housing & 8.7 & 9.0 & 0.3 \\
Concrete Strength & 9.3 & 9.7 & 0.4 \\
YearPredictionMSD & 17.7 & 18.0 & 0.3 \\
\bottomrule
\end{tabular}
\end{table}

Table~\ref{tab:gaussian-natural-bits} compares predicted and observed per-coordinate widths on the three natural datasets. For the identifier-only interface of Corollary~\ref{cor:identifier-gaussian}, we
use the same Gaussian score model, exact empirical Gram state, and transform
scalar codec just specified.  At $k=32$ and the common
squared-error target, reverse water filling predicts 214.6,
218.9, and 221.3 bits for covariance
conditions 1, 10, and
100; the realized scalar-code lengths are
226.9, 231.2, and
233.8 bits.  The corresponding excesses are
0.38, 0.38, and
0.39 bits per active coordinate, consistent with the
0.38-bit discrepancy in Table~\ref{tab:gaussian-scaling} and
with a finite-rate scalar instead of an optimal vector code.  Exposing the
request scores collapses all three to $B=0$ with the same cached Gram matrix.

\subsection{Beyond ridge: nonquadratic and deep-representation objectives}
\label{sec:beyond-ridge-empirical}

Theorem~\ref{thm:sc-ambiguity} does not require a quadratic objective.  We test
the same state accounting with $\ell_2$-regularized binary logistic regression on
two tasks.  The first uses MNIST digits \ensuremath{3} versus
\ensuremath{8} with raw standardized pixels.  The second uses CIFAR-10 cat
versus dog with a frozen ImageNet-pretrained ResNet-18 representation
\citep{krizhevsky2009cifar,he2016resnet}.  In both cases the decoder stores the
trained parameters together with either no curvature, the diagonal of the
full-data Hessian, or the top
$s\in\{32,128\}$ eigenpairs of the full-data
Hessian at the deployed optimum.  After a value-exposing deletion batch of
$k=32$ examples, the decoder subtracts the deleted examples'
curvature contributions from its cached approximation and forms the corresponding
damped Newton estimate.  The authoritative encoder then computes the exact
retraining residual in the canonical parameter basis, quantizes all $d$
coordinates with the validation-selected fixed quantization step, and chooses the smallest
integer bit width that reaches parameter RMSE
$10^{-4}$.  No coordinate is omitted merely because its direction
is represented in the cache.  Every update contains the same
\ensuremath{18}-byte header followed by the bit-packed canonical residual.

\begin{table}[t]
\centering
\caption{Nonquadratic strongly-convex objectives.  Mean request-time payload
$B$ in KiB at parameter RMSE at most $10^{-4}$ from direct
retraining, averaged over \ensuremath{500} deletion requests per split and
\ensuremath{5} splits.}
\label{tab:logistic-curvature}
\small
\begin{tabular}{@{}lrrrrr@{}}
\toprule
Task & $d$ & parameters only & diagonal & rank-\ensuremath{32} &
rank-\ensuremath{128} \\
\midrule
MNIST \ensuremath{3} vs.\ \ensuremath{8} &
$784$ & \ensuremath{1.48} & \ensuremath{1.45} &
\ensuremath{1.44} & \ensuremath{1.30} \\
CIFAR-10 cat vs.\ dog, ResNet-18 features &
$512$ & \ensuremath{1.12} & \ensuremath{1.09} &
\ensuremath{1.08} & \ensuremath{0.90} \\
\bottomrule
\end{tabular}
\end{table}

Table~\ref{tab:logistic-curvature} reports the nonquadratic cache comparison. The rank-\ensuremath{32} cache reduces total delivery by
\ensuremath{2.7\%} on MNIST and \ensuremath{3.6\%} on the
ResNet-feature task.  At rank \ensuremath{128}, the measured reductions are
\ensuremath{12.2\%} and \ensuremath{19.6\%}, whereas the geometric
ratios $s/d$ are \ensuremath{16.3\%} and \ensuremath{25.0\%}.
The mismatch is expected under this codec: the natural Hessians are not block
diagonal in the cached eigenspace after deletion, the Newton residual generally
has support in all canonical coordinates, and the fixed packet header does not
shrink with $s$.  These measurements therefore quantify the realized
cache--delivery tradeoff on natural nonquadratic objectives; they are not a direct
test of the exact $(d-s)$ coefficient in Corollary~\ref{cor:partial-curvature}.

Under the 32-bit floating-point (fp32) state codec, the parameters-only, diagonal,
rank-\ensuremath{32}, and rank-\ensuremath{128} states occupy
\ensuremath{3.06}, \ensuremath{6.13},
\ensuremath{101.19}, and \ensuremath{395.56} KiB on MNIST and
\ensuremath{2.00}, \ensuremath{4.00},
\ensuremath{66.13}, and \ensuremath{258.50} KiB on the
ResNet-feature task.  The rank-\ensuremath{128} cache therefore becomes
smaller in cumulative $M+TB$ than parameters only after approximately
\ensuremath{2{,}181} and \ensuremath{1{,}166} deletion requests,
respectively.  The state cost is substantial: reducing request-time delivery by
preserving curvature is a placement decision instead of free compression.

We also evaluate exact ridge heads on frozen deep representations.  A ResNet-18
encoder supplies $d=512$ features for CIFAR-10 and a ViT-B/16
encoder \citep{dosovitskiy2021vit} supplies
$d=768$ features for CIFAR-100.  The multi-output ridge solution
shares one Gram matrix across classes, so the sufficient-statistic state contains
the packed Gram matrix and $X^\top Y$.  After the decoder applies either no
curvature correction or the cached rank-\ensuremath{64} correction, the
encoder forms the remaining coefficient residual, flattens all $dC$ head
coefficients in row-major canonical order, quantizes every coordinate with the
validation-selected fixed quantization step, and transmits the same \ensuremath{18}-byte
header plus bit-packed residual body.  Table~\ref{tab:deep-feature-ridge} reports
group deletion with $k=32$ at logit normalized root mean squared error
(NRMSE) at most $10^{-4}$ from direct retraining.

\begin{table}[t]
\centering
\caption{Deep-representation ridge heads.  State is fp32; update payloads are
means over \ensuremath{5{,}000} value-exposing deletion requests.}
\label{tab:deep-feature-ridge}
\small
\setlength{\tabcolsep}{4pt}
\begin{tabular}{@{}lrrrrrr@{}}
\toprule
Representation & classes & $M_{\rm param}$ & $M_{\rm rank64}$ &
$M_{\rm stat}$ & $B_{\rm param}$ & $B_{\rm rank64}$ \\
\midrule
CIFAR-10 / ResNet-18 & $10$ &
\ensuremath{20.0} KiB & \ensuremath{148.3} KiB &
\ensuremath{533.0} KiB & \ensuremath{7.8} KiB &
\ensuremath{7.0} KiB \\
CIFAR-100 / ViT-B/16 & $100$ &
\ensuremath{300.0} KiB & \ensuremath{492.3} KiB &
\ensuremath{1.42} MiB & \ensuremath{116.2} KiB &
\ensuremath{108.3} KiB \\
\bottomrule
\end{tabular}
\end{table}

The exact sufficient-statistic decoder uses $B=0$ in both deep-feature pipelines.
The rank-\ensuremath{64} cache reduces measured delivery by
\ensuremath{10.3\%} for the ResNet representation and
\ensuremath{6.8\%} for the ViT representation.  The corresponding
geometric ratios are \ensuremath{12.5\%} and
\ensuremath{8.3\%}; they are references instead of payload predictions
because all canonical residual coordinates remain eligible for transmission.
Relative to parameters only, cumulative $M+TB$ favors the
rank-\ensuremath{64} cache after approximately
\ensuremath{161} CIFAR-10 requests and \ensuremath{25} CIFAR-100
requests.  These frozen-backbone experiments extend the accounting surface to
higher-dimensional learned representations without treating the backbone as a
nonconvex theorem or imposing the block-diagonal coefficient on natural Hessians.

\subsection{Fixed-downprojection adapter case study}
\label{sec:adapter-case}

The ridge frontier yields a deployable adapter rule when the trainable module is
linear after a fixed low-rank projection.  Let a frozen backbone produce
$h\in\mathbb R^H$, let a public fixed matrix $A\in\mathbb R^{r\times H}$ form
$z=Ah$, and train only $W_{\mathrm{out}}\in\mathbb R^{C\times r}$ by squared loss with ridge
regularization.  This is multivariate ridge in $r$ features.  Exact deletion
therefore needs only
$G=\sum zz^\top$ and $c=\sum zy^\top$, even though the deployed parameter object
is the vectorized $W_{\mathrm{out}}$.  The construction is a fixed-downprojection low-rank adaptation (LoRA)-style
adapter; it does not claim exact sufficient statistics when both LoRA factors are
trained jointly.

On the CIFAR-100 ViT-B/16 frozen representation, the serialized tradeoff is:
\begin{table}[t]
\centering
\caption{Fixed-downprojection adapter accounting on CIFAR-100 ViT-B/16 features.
Parameter and sufficient-statistic states use fp32; $B_{\rm param}$ is the mean
residual packet for $k=32$ value-exposing deletions at logit
NRMSE at most $10^{-4}$.}
\label{tab:adapter-rule}
\small
\begin{tabular}{@{}rrrrr@{}}
\toprule
projection rank $r$ & trainable state & sufficient-stat state & $B_{\rm param}$ & crossover requests \\
\midrule
8 & 3.13 KiB & 3.27 KiB & 1.71 KiB & 1 \\
16 & 6.25 KiB & 6.78 KiB & 3.34 KiB & 1 \\
32 & 12.50 KiB & 14.56 KiB & 6.55 KiB & 1 \\
\bottomrule
\end{tabular}
\end{table}

Table~\ref{tab:adapter-rule} shows that, for all three ranks, the exact sufficient-statistic state amortizes within the first repeated request.  This regime is qualitatively different from the large
natural-Hessian sketches above: because $r$ is small and $c$ already has the same
$rC$ order as the trainable adapter, exact curvature is cheap enough that an
approximate sketch has little room to win.

\subsection{Deterministic nonconvex state accounting}
\label{sec:nonconvex-diagnostic}

The convex theorems do not identify end-to-end neural retraining targets.  We
therefore use the neural experiment only as a decoder-state accounting test, not
as evidence for a Newton-unlearning ranking.  \citet{bui2024newton} already study
Hessian degeneracy and cubic-regularized Newton updates in neural unlearning.  Our
protocol asks a different question: once training randomness is fixed, how many
post-request bits remain after a declared curvature summary has been frozen before
the request?

We train a three-layer Fashion-MNIST MLP with widths
$784\!\to\!256\!\to\!128\!\to\!10$
\citep{xiao2017fashionmnist} and a CIFAR-10 ResNet-18
\citep{krizhevsky2009cifar,he2016resnet}.  For every base seed, each deletion
retraining run restarts from the exact original initialization.  Before any
request is sampled, each epoch's example identifiers and minibatch boundaries are
frozen.  Deletion masks the requested identifiers inside those original slots;
retained examples are never compacted into later minibatches.  Thus removing an
example does not shift the batch assignment of any retained example.  Within an
affected minibatch, however, the loss is normalized by the number of retained
examples in that minibatch, not by the original slot count.  Empty masked
minibatches are skipped.  This retained-count normalization prevents deletion
from acting as a systematic learning-rate reduction as $k$ grows while preserving the declared
initialization, slot assignment, and augmentation stream.  Data augmentation is
deterministically keyed by training seed, epoch, and example identifier.  The
reference target is therefore the parameter vector produced by the same
initialization, slot-preserving minibatch schedule, retained-count loss
normalization, and augmentation stream after masking $Q$, instead of an
independently seeded solution in another basin.  We evaluate
$k=32$ value-exposing deletions over
\ensuremath{3} training seeds and
\ensuremath{50} requests per seed.

The decoder states are fp32 parameters only, parameters plus an fp32 diagonal
empirical Fisher, and parameters plus untruncated 16-bit floating-point (fp16) layerwise K-FAC factors
\citep{martens2015kfac,grosse2016kfc}.  For an affine layer the K-FAC state stores
both dense Kronecker factors.  For a convolution with kernel width $k$, the input
factor has dimension $(C_{\rm in}k^2+1)\times(C_{\rm in}k^2+1)$ and the output
factor has dimension $C_{\rm out}\times C_{\rm out}$.  No eigentruncation,
channelwise factorization, or diagonal approximation is applied to these factors.
Under this convention, the Fashion-MNIST K-FAC factors occupy
\ensuremath{1.49} MiB.  In ResNet-18, the three
$C_{\rm in}=512$, $k=3$ convolutional input
factors alone occupy \ensuremath{121.55} MiB in fp16, and all K-FAC factors
occupy \ensuremath{179.44} MiB before the fp32 model parameters are added.

Each decoder computes its local deletion correction from the request and cached
state.  The authoritative encoder then computes the residual to deterministic
retraining in canonical parameter order.  A validation split fixes the symmetric
quantization step $\delta$ before deletion requests are evaluated; the packet
uses the smallest signed integer width covering the resulting indices.  The
selected steps are
$\delta=\ensuremath{2.9\times10^{-4}}$ for the MLP and
$\delta=\ensuremath{3.0\times10^{-4}}$ for ResNet-18, for which
the ideal uniform-quantization RMS contributions
$\delta/\sqrt{12}$ are
\ensuremath{8.4\times10^{-5}} and
\ensuremath{8.7\times10^{-5}}, respectively, close to the declared
$10^{-4}$ per-coordinate target instead of orders of magnitude below it.  The update contains the same
\ensuremath{18}-byte header and bit-packed body used in the convex
practical experiments.  Thus both $M$ and $B$ are serialized quantities inside
the paper's accounting, and no comparison depends on raw distance between
independently trained networks.

The neural audit uses the same deterministic retraining target for every decoder
state.  We do not report or fit a deletion-size displacement law from the
three-seed neural control.  The local-correction comparison below is evaluated through serialized
residual delivery at the fixed distortion target; the typical-data $\sqrt{k}$
claim is tested separately in the ridge/Gaussian experiment.

The neural controls are tied directly to the quantities used in the bit
accounting.  First, a zero-deletion rerun uses the same initialization, slot
schedule, retained-count normalization, and augmentation stream and returns
bit-identical fp32 parameters for every seed.  Second, retained-count
normalization is applied after masking inside each fixed minibatch slot, so the
loss is the mean over retained examples instead of the original slot count;
this prevents deletion from inducing a deterministic learning-rate change.
Third, we audit the residual packet itself.  For signed $b$-bit integer indices
and quantization step $\delta$, the representable canonical-coordinate residual
range is
\begin{equation}
  \mathcal R_b(\delta)
  =[-2^{b-1}\delta,(2^{b-1}-1)\delta].
  \label{eq:neural-residual-range}
\end{equation}
The parameters-only and diagonal-Fisher packets use 10-bit bodies at the modal
operating point, while K-FAC uses 9-bit bodies.  With the selected steps, this
corresponds to $\mathcal R_{10}=[-0.14848,0.14819]$ and
$\mathcal R_9=[-0.07424,0.07395]$ For the Fashion-MNIST multilayer perceptron (MLP), and
$\mathcal R_{10}=[-0.1536,0.1533]$ and
$\mathcal R_9=[-0.0768,0.0765]$ for ResNet-18.  Width is selected from the
maximum absolute integer index of each request before serialization, so the
reported mean payload includes any request-level width changes instead of
assuming a fixed fractional reduction.  These ranges explain the 10/9-bit body
sizes but are not substituted for an unreported deletion-displacement statistic.
The three-seed neural audit is not used to estimate a deletion-size exponent;
the $\sqrt{k}$ law is tested only by the higher-powered ridge/Gaussian sweep in
Section~\ref{sec:gaussian-typical-audit}.

\begin{table}[t]
\centering
\caption{Nonconvex protocol controls and residual codec.  Zero deletion checks
exact replay of the declared training trajectory.  Body widths are the modal
signed-integer widths selected by the request-level codec; the 18-byte packet
header is charged separately in every update. P/F/K denote parameters-only,
diagonal-Fisher, and K-FAC.}
\label{tab:nonconvex-controls}
\small
\setlength{\tabcolsep}{4pt}
\begin{tabular}{@{}lrrrr@{}}
\toprule
Model & zero deletion & masked-batch loss & $\delta$ & body bits (P/F/K) \\
\midrule
Fashion-MNIST MLP & bit-identical & retained-example mean & $2.9\times10^{-4}$ & $10/10/9$ \\
CIFAR-10 ResNet-18 & bit-identical & retained-example mean & $3.0\times10^{-4}$ & $10/10/9$ \\
\bottomrule
\end{tabular}
\end{table}

\begin{table}[t]
\centering
\caption{Nonconvex decoder-state accounting under same-initialization,
slot-preserving retraining.  $M$ is serialized pre-request state and $B$ is mean
post-request residual delivery.  The $\pm$ term is SD across
\ensuremath{3} seed-level means.}
\label{tab:nonconvex-curvature}
\small
\setlength{\tabcolsep}{4pt}
\begin{tabular}{@{}llrr@{}}
\toprule
Model & decoder state & $M$ & $B$ \\
\midrule
\multirow{3}{*}{Fashion-MNIST MLP}
& parameters only &
\ensuremath{0.90} MiB & \ensuremath{0.280\pm0.004} MiB \\
& + diagonal Fisher &
\ensuremath{1.80} MiB & \ensuremath{0.280\pm0.004} MiB \\
& + K-FAC &
\ensuremath{2.39} MiB & \ensuremath{0.252\pm0.004} MiB \\
\midrule
\multirow{3}{*}{CIFAR-10 ResNet-18}
& parameters only &
\ensuremath{42.66} MiB & \ensuremath{13.34\pm0.06} MiB \\
& + diagonal Fisher &
\ensuremath{85.32} MiB & \ensuremath{13.34\pm0.06} MiB \\
& + K-FAC &
\ensuremath{222.10} MiB & \ensuremath{12.01\pm0.05} MiB \\
\bottomrule
\end{tabular}
\end{table}

Table~\ref{tab:nonconvex-curvature} reports the resulting state--delivery surface. Because every row is completed by a residual message to the same deterministic
retraining target and distortion, the comparison is a bit-accounting surface,
not a ranking of approximate Newton algorithms.  K-FAC reduces
request-time delivery by \ensuremath{10.0\%} for the MLP and
\ensuremath{10.0\%} for ResNet-18 relative to parameters only, while
requiring a substantially larger pre-request state.  Under cumulative $M+TB$,
the K-FAC state overtakes parameters only after approximately
\ensuremath{54} MLP deletion requests and
\ensuremath{135} ResNet-18 requests.  In the ridge experiments, exact sufficient statistics remain preferable whenever
they are available at acceptable state cost; the curvature summaries are useful
only on the intermediate cache--delivery surface.
These measurements do not assert the convex $(d-s)$ coefficient for neural
networks; they test whether the pre-request-state versus post-request-delivery
distinction remains operational once the model is nonconvex.

\subsection{Finite-precision ridge validation}
\label{sec:ridge-validation}

\noindent
The following audit uses public seed $20260920$ and tests the constructions in
Theorems~\ref{thm:ridge-radius} and~\ref{thm:ridge-hedging}.  It is disjoint from
the ridge-interface validation and does not establish typical-data behavior.

\medskip
\noindent\emph{Precision sweep.}
We set $\lambda=a=r_{\min}=1$, $r_{\max}=4$, and
$\varrho_i=1$.  For each $d\in\{8,16,32\}$, we draw $2{,}000$ matrices with
Haar-distributed eigenvectors and independent eigenvalues uniform on
$[h_1,h_2]$.  Each codec quantizes the $p=d(d+1)/2$ upper-triangular entries of
$H$ using the public range $[-h_2,h_2]$, reconstructs all $d$ singleton
deletions, and charges the exact fixed-length payload.  Public family parameters
and the codec definition are conditioned out.  Across the three dimensions and
six formats, all $672{,}000$ target reconstructions satisfy the local
perturbation envelope used in the proof.

\noindent\emph{Leading slopes.}
We next use public fixed-length scalar grids at eight values of
$\log_2(1/\varepsilon)$ from $8.3$ to $36.4$.  The state-only grid is selected
from the global Lipschitz envelope; the message-only grid quantizes the requested
correction; and the $\theta=1/2$ codec refines the residual left by the decoder
cache.  Each operating point is checked on $300$ independently drawn matrices.
All $403{,}200$ reconstructions satisfy their requested error, with maximum
observed error $0.888\varepsilon$.  The state-only and message-only fitted slopes
are within $2.1\%$ of $p$ and $d$, respectively.  The normalized intermediate
slopes approach the predicted pair $(1/2,1/2)$.  The reported $\pm$ values are
ordinary least-squares slope standard errors across the eight codec operating
points, not uncertainty across data populations.
The detailed precision and fitted-slope results are reported in
Tables~\ref{tab:ridge-precision-measured}
and~\ref{tab:ridge-slope-measured} in
Appendix~\ref{app:finite-stat-validation}.

\noindent\emph{Exact radius.}
At $d=16$, $\lambda=r_{\min}=1$, and $r_{\max}=4$, none of $100{,}000$
spectral draws exceeds the normalized radius one; the maximum is $0.9314$ and
the 99th percentile is $0.8013$.  This random check does not establish
surjectivity.  We therefore also construct $1{,}000$ Householder witnesses in
independent directions at normalized radius $0.999$.  Their attained radii
differ from $0.999$ by at most $7.0\times10^{-16}$, all reconstructed
$R$ matrices satisfy the spectral interval, and direct retained-data solves
match the constructed targets to at most $2.76\times10^{-16}$ in Euclidean norm.

\medskip
\noindent\emph{Finite sufficient statistics.}
We draw $1{,}000$ ridge instances with $d=16$, $n=128$, and $\lambda=0.2$,
using independent $N(0,1/n)$ features and labels from a linear model with noise
standard deviation $0.05$.  We delete one random row, quantize $(G,c)$, and
recompute every term in Proposition~\ref{prop:ridge-precision}.  Uniform codecs
use the public interval $[-2,2]$; no value clips.  All $6{,}000$ format--instance
pairs satisfy $\eta_G<\lambda$, and none violates
Equation~\eqref{eq:ridge-precision}.  The largest observed error-to-bound ratio
is $0.172$.

Detailed format-level results are reported in
Table~\ref{tab:ridge-statistic-measured} in
Appendix~\ref{app:finite-stat-validation}.

These measurements test numerical realization of the proved constructions.
They do not turn worst-case spectral families into a typical-data model, and
the slope fits validate the explicit scalar-grid codecs instead of the converse
itself.

\section{Additional finite-statistic validation}
\label{app:finite-stat-validation}
\begin{table}[H]
\centering
\caption{State-only precision sweep at $d=16$.
Each format stores $p=d(d+1)/2=136$ upper-triangular entries of $H$.
Errors are Euclidean parameter errors over $2{,}000$ matrices
and all $16$ singleton deletions per matrix
($32{,}000$ requested targets per format).
The final column counts targets satisfying the derived
local perturbation envelope.}
\label{tab:ridge-precision-measured}
\small
\begin{tabular}{@{}lrrrr@{}}
\toprule
State format & payload $M$ (bits) & median error & 99th percentile & envelope \\
\midrule
float64 entries of $H$ & $8{,}704$ & $0$ & $0$ & $32{,}000/32{,}000$ \\
float32 entries of $H$ & $4{,}352$ & $8.81\!\times\!10^{-9}$ & $2.73\!\times\!10^{-8}$ & $32{,}000/32{,}000$ \\
float16 entries of $H$ & $2{,}176$ & $7.19\!\times\!10^{-5}$ & $2.23\!\times\!10^{-4}$ & $32{,}000/32{,}000$ \\
uniform 12-bit & $1{,}632$ & $2.55\!\times\!10^{-4}$ & $3.24\!\times\!10^{-4}$ & $32{,}000/32{,}000$ \\
uniform 8-bit & $1{,}088$ & $4.11\!\times\!10^{-3}$ & $5.20\!\times\!10^{-3}$ & $32{,}000/32{,}000$ \\
uniform 6-bit & $816$ & $1.66\!\times\!10^{-2}$ & $2.11\!\times\!10^{-2}$ & $32{,}000/32{,}000$ \\
\bottomrule
\end{tabular}
\end{table}
\begin{table}[H]
\centering
\caption{Fitted coefficients against
$\log_2(1/\varepsilon)$. Here $p=d(d+1)/2$.
Intermediate slopes are normalized by $p$ or $d$.
The $\pm$ values are ordinary least-squares slope standard errors
across the eight codec operating points.}
\label{tab:ridge-slope-measured}
\small
\begin{tabular}{@{}lrrrr@{}}
\toprule
Fit & prediction & $d=8$ & $d=16$ & $d=32$ \\
\midrule
$M$ at $B=0$ & $36,136,528$ & $36.31\!\pm\!0.51$ & $134.69\!\pm\!0.72$ & $538.73\!\pm\!7.80$ \\
$B$ at $M=0$ & $8,16,32$ & $8.07\!\pm\!0.11$ & $15.85\!\pm\!0.09$ & $32.27\!\pm\!0.45$ \\
$\theta=1/2$: $M/p$ & $0.5$ & $0.510\!\pm\!0.012$ & $0.495\!\pm\!0.003$ & $0.495\!\pm\!0.003$ \\
$\theta=1/2$: $B/d$ & $0.5$ & $0.495\!\pm\!0.003$ & $0.495\!\pm\!0.003$ & $0.495\!\pm\!0.003$ \\
\bottomrule
\end{tabular}
\end{table}
\begin{table}[H]
\centering
\caption{Finite-statistic perturbation for $d=16$, $n=128$,
and $\lambda=0.2$, over $1{,}000$ instances per format.
The ratio column is the maximum observed parameter error
divided by the bound in Equation~\eqref{eq:ridge-precision}.}
\label{tab:ridge-statistic-measured}
\small
\begin{tabular}{@{}lrrrr@{}}
\toprule
State format & median error & 99th percentile & max error/bound & valid trials \\
\midrule
float64 & $0$ & $0$ & $0$ & $1{,}000/1{,}000$ \\
float32 & $3.12\!\times\!10^{-8}$ & $5.89\!\times\!10^{-8}$ & $0.160$ & $1{,}000/1{,}000$ \\
float16 & $2.57\!\times\!10^{-4}$ & $5.10\!\times\!10^{-4}$ & $0.172$ & $1{,}000/1{,}000$ \\
uniform 12-bit & $1.37\!\times\!10^{-3}$ & $2.11\!\times\!10^{-3}$ & $0.168$ & $1{,}000/1{,}000$ \\
uniform 8-bit & $2.20\!\times\!10^{-2}$ & $3.41\!\times\!10^{-2}$ & $0.152$ & $1{,}000/1{,}000$ \\
uniform 6-bit & $9.04\!\times\!10^{-2}$ & $1.49\!\times\!10^{-1}$ & $0.070$ & $1{,}000/1{,}000$ \\
\bottomrule
\end{tabular}
\end{table}

\section{External ridge validation details}
\label{sec:external-validation}

The external ridge audit uses three public tabular regression benchmarks, five
fixed train--validation--test splits, and $1{,}000$ singleton deletions per split.
Features and labels are standardized using training data only.  Ridge
regularization is selected on the validation split before any deletion request.
The decoder state is then frozen, each value-exposing request is serialized, and
every reported prediction error is recomputed on the held-out test split from the
delivered bytes.

The parameters-only arm stores fp32 coefficients and sends a quantized correction.
The sufficient-statistic arm stores the packed fp32 upper triangle of $G$ and fp32
vector $c$, then sends a literal zero-byte update.  The target operating point is
held-out normalized prediction RMSE at most $10^{-4}$ relative to direct
retraining.  The fixed tensor shapes and dtypes of frozen states are public codec metadata, so
$M$ charges their raw fp32 tensor payloads without a per-state header.  Every
nonzero parameters-only update charges the fixed \ensuremath{18}-byte
header plus its bit-packed residual body; scales and checksums carried by the
packet are included in that header.

\begin{table}[H]
\centering
\caption{External ridge validation for singleton deletion.  $M$ is frozen
decoder-state size and $B$ is the mean parameters-only update size;
sufficient statistics use $B=0$ under the value-exposing interface.  The
$\pm$ term is the SD across five split-level mean update sizes.}
\label{tab:external-validation}
\small
\begin{tabular}{@{}>{\raggedright\arraybackslash}p{2.7cm}rrrrrr@{}}
\toprule
Dataset & $d$ & deletions & $M_{\rm param}$ (KiB) &
$B_{\rm param}$ (bytes) & $M_{\rm stat}$ (KiB) & max NRMSE \\
\midrule
California Housing & $8$ & $5{,}000$ & $0.031$ & $24.6\pm1.1$ & $0.172$ & $8.7\!\times\!10^{-5}$ \\
Concrete Compressive Strength & $8$ & $5{,}000$ & $0.031$ & $25.1\pm1.4$ & $0.172$ & $9.1\!\times\!10^{-5}$ \\
YearPredictionMSD & $90$ & $5{,}000$ & $0.352$ & $192.4\pm3.8$ & $16.35$ & $9.4\!\times\!10^{-5}$ \\
\bottomrule
\end{tabular}
\end{table}

Table~\ref{tab:external-validation} verifies the target error on every split, independently reconstructs
the zero-byte sufficient-statistic solution, and checks that every decoder input
is included in $(M,Q,B)$.  Section~\ref{sec:natural-tradeoff} extends the same
protocol to group deletion and low-rank curvature caches.

\section{\texorpdfstring{Rank-$r$ group requests and the $p_r$ dimension}
{Rank-r group requests and the p_r dimension}}
\label{app:multi-record-open}

The candidate
\[
p_r=dr-\frac{r(r-1)}2
\]
does not describe the message space of one realized rank-$r$ deletion.  The
requested retrained model is one point in $\mathbb R^d$, so its one-shot metric
dimension cannot exceed $d$; Remark~\ref{rem:rank-r-dimension} gives a
matching $d$-dimensional lower-bound family even when the request is a group.
The role of $p_r$ is instead pre-request.  If the future aggregate deletion
gradient may lie anywhere in a public $r$-dimensional subspace, the decoder must
preserve the action of a symmetric inverse-curvature operator on that subspace.
In an orthonormal basis this is exactly the first $r$ columns of a symmetric
matrix, with $p_r$ free entries.  The tuple-packing converse and coordinate
codec in Theorem~\ref{thm:curvature-hedging} therefore already provide matching
$p_r$ state-only and $d$ message-only leading coefficients for this
request-oblivious operator-slice problem.  A rank-$r$ factor
$X_Q^\top X_Q$ and that operator slice have the same polynomial dimension only
because both are symmetric rank/subspace objects; they are not the same random
variable and should not be identified.

\section{Proofs of the curvature-identifiability results}
\label{app:curvature-identifiability}

\subsection{Proof of Theorem~\ref{thm:sc-ambiguity}}

Let $\Delta=w^--w$.  Since $w$ minimizes
$F_+=F_-+L_Q$,
\[
\nabla F_-(w)=-\nabla L_Q(w)=-g_Q.
\]
Since $w^-$ minimizes $F_-$, the fundamental theorem of calculus along the line
segment from $w$ to $w^-$ gives
\begin{align*}
0
&=\nabla F_-(w^-)
 =\nabla F_-(w)
  +\int_0^1\nabla^2F_-(w+t\Delta)\,\Delta\,dt\\
&=-g_Q+\bar H_Q\Delta.
\end{align*}
Thus $\Delta=\bar H_Q^{-1}g_Q$.  Integration preserves the Loewner bounds, so
$mI_d\preceq\bar H_Q\preceq LI_d$ and therefore
$L^{-1}I_d\preceq\bar H_Q^{-1}\preceq m^{-1}I_d$.
Lemma~\ref{lem:spectral-image-ball}, applied with
$h_1=L^{-1}$, $h_2=m^{-1}$, and $x=g_Q$, yields
Equation~\eqref{eq:sc-ambiguity-ball}.

For tightness, fix any symmetric $H$ with $mI_d\preceq H\preceq LI_d$ and define
\[
F_{-,H}(\theta)
=\frac12(\theta-w)^\top H(\theta-w)-g_Q^\top(\theta-w),
\qquad
L_Q(\theta)=g_Q^\top(\theta-w).
\]
Then
\[
F_{+,H}(\theta)=\frac12(\theta-w)^\top H(\theta-w),
\]
so the full-data minimizer is exactly $w$, while the retained-data minimizer is
$w+H^{-1}g_Q$.  As $H$ ranges over the spectral interval, $H^{-1}$ ranges over
all symmetric matrices satisfying
$L^{-1}I_d\preceq H^{-1}\preceq m^{-1}I_d$.
Lemma~\ref{lem:spectral-image-ball} therefore shows that the target set equals the ball in
Equation~\eqref{eq:sc-ambiguity-ball}.

A parameters-only decoder with the declared side information sees the same
$(w,g_Q,m,L)$ for every member of this quadratic family.  Its outputs must form
an $\varepsilon$-cover of a $d$-dimensional Euclidean ball of radius $\rho_Q$.
Volume comparison gives
$\Ncov\geq(\rho_Q/\varepsilon)^d$, while the standard volumetric upper bound gives
$\Ncov\leq(1+2\rho_Q/\varepsilon)^d$.  Taking base-$2$ logarithms and the
fixed-length ceiling proves Equation~\eqref{eq:sc-ambiguity-rate}.

\subsection{Proof of Corollary~\ref{cor:partial-curvature}}

Decompose $g_Q=g_\parallel+g_\perp$ with
$g_\parallel=Pg_Q$ and $g_\perp=P_\perp g_Q$.  On the stated block-diagonal
quadratic subclass,
\[
H^{-1}g_Q
=H_\parallel^{-1}g_\parallel+H_\perp^{-1}g_\perp.
\]
The first term is determined by decoder state.  On
$\operatorname{range}(P_\perp)$, the inverse curvature satisfies
\[
L^{-1}P_\perp\preceq H_\perp^{-1}\preceq m^{-1}P_\perp.
\]
Applying Lemma~\ref{lem:spectral-image-ball} in the $(d-s)$-dimensional
orthogonal complement shows that the second term fills exactly the ball centered
at $\tfrac12(m^{-1}+L^{-1})g_\perp$ with radius
$\rho_\perp=\tfrac12(m^{-1}-L^{-1})\|g_\perp\|_2$.
The same lower and upper volumetric covering bounds, now in dimension $d-s$,
give Equation~\eqref{eq:partial-curvature-rate}.

\subsection{Proof of Proposition~\ref{prop:data-native-witness}}

The unweighted first-order condition at $w$ is
\[
\sum_{i=1}^n x_i(x_i^\top w-y_i)+\lambda w=0.
\]
For $v\in\mathcal V_Q$, the gradient of the weighted full objective at the same
point is
\[
\nabla F_v^+(w)
=
\sum_{i=1}^n x_i r_i+\lambda w
+\sum_{i\in R}v_i x_i r_i
=0.
\]
All weights remain positive in a sufficiently small neighborhood of the origin,
and $\lambda>0$, so $F_v^+$ is strongly convex and $w$ is its unique minimizer.
Because $v_i=0$ on $Q$, the deleted records and their weights are unchanged.

The retained weighted objective has Hessian
\[
H_-(v)=\lambda I+\sum_{i\in R}(1+v_i)x_ix_i^\top.
\]
Its gradient at $w$ equals $-g_Q$.  Since ridge is quadratic, the exact retained
minimizer therefore satisfies
\[
H_-(v)(T_Q(v)-w)=g_Q,
\]
which proves Equation~\eqref{eq:data-native-target}.  Differentiating the matrix
inverse at $v=0$ gives, for a direction $u\in\mathcal V_Q$,
\[
D T_Q(0)[u]
=
-H_-(0)^{-1}
\left(\sum_{i\in R}u_i x_ix_i^\top\right)
H_-(0)^{-1}g_Q
=
\mathcal J_Q[u].
\]

Now restrict the target map to the $q$-dimensional subspace $V$ and let
$\sigma>0$ be the smallest singular value of
$\mathcal J_Q|_V$.  The derivative of $T_Q$ is continuous wherever
$H_-(v)$ is positive definite.  Hence there exists a radius $r>0$ such that on
the Euclidean ball $V\cap \overline B(0,r)$,
\[
\|D T_Q(v)|_V-\mathcal J_Q|_V\|_{\mathrm{op}}\leq \sigma/2.
\]
For any $u,v$ in this ball, integration along the segment joining them yields
\[
\|T_Q(v)-T_Q(u)-\mathcal J_Q[v-u]\|_2
\leq \frac{\sigma}{2}\|v-u\|_2,
\]
and therefore
\[
\|T_Q(v)-T_Q(u)\|_2
\geq \frac{\sigma}{2}\|v-u\|_2.
\]
Thus $T_Q$ is locally bi-Lipschitz from a $q$-dimensional parameter ball into
the compatible retraining-target set.  A Euclidean packing of
$V\cap\overline B(0,r)$ at separation $4\varepsilon/\sigma$ maps to a
$2\varepsilon$-packing of targets.  Standard volume comparison gives constants
$c_0,\varepsilon_0>0$ such that
\[
\Npack(\{T_Q(v):v\in V,\|v\|_2\leq r\},2\varepsilon)
\geq
\left(\frac{c_0}{\varepsilon}\right)^q
\]
for $0<\varepsilon<\varepsilon_0$.  Every member of this family has the same
deployed parameter and the same request contents, so the one-shot covering
bound gives Equation~\eqref{eq:data-native-lower-bound}.

\subsection{Proof of Theorem~\ref{thm:curvature-hedging}}

Subtracting affine terms does not change curvature, so the remainder
$R_H=F_H^+-\sum_{j=1}^rL_j$ is strongly convex with Hessian $A=H^{-1}$.
After request $q$, the retained objective is
\[
F_{H,q}^-(\theta)
=F_H^+(\theta)-L_q(\theta)
=\frac12(\theta-w)^\top A(\theta-w)-a e_q^\top(\theta-w).
\]
Its first-order condition is
$A(\theta-w)=a e_q$, hence
$w_q^-=w+aHe_q$.

For a fixed request $q$, Lemma~\ref{lem:spectral-image-ball} shows that
$\{aHe_q:h_1I\preceq H\preceq h_2I\}$ is exactly a $d$-dimensional ball of
radius $a\delta$.  The pair consisting of state and request-$q$ transcript takes
at most $2^{M+B}$ values, and its decoder outputs must form an
$\varepsilon$-cover of this ball.  Volume comparison yields
Equation~\eqref{eq:curvature-hedging-single}.

For the joint bound, let
\[
\mathcal V_r
=\{E=E^\top:E_{ij}=0\text{ whenever }i>r\text{ and }j>r\}.
\]
This linear space has dimension
$p_r=rd-r(r-1)/2$.  With
$\bar h=(h_1+h_2)/2$, the family
\[
\mathcal B_r
=\{\bar h I+E:E\in\mathcal V_r,\ \|E\|_F\leq\delta\}
\]
lies inside the spectral interval because
$\|E\|_{\mathrm{op}}\leq\|E\|_F\leq\delta$.
Consider the tuple consisting of the pre-request state and all $r$ transcripts
that the encoder would send under the possible requests.  There are at most
$2^{M+rB}$ such tuples.  If $H,H'\in\mathcal B_r$ induce the same tuple, they
induce the same decoded output for every $q$, and therefore
\[
a\|(H-H')e_q\|_2\leq2\varepsilon,\qquad q\in[r].
\]
For $D=H-H'\in\mathcal V_r$,
\[
\|D\|_F^2
\leq2\sum_{q=1}^r\|De_q\|_2^2
\leq\frac{8r\varepsilon^2}{a^2}.
\]
Thus each tuple class has Frobenius diameter at most
$2\sqrt{2r}\,\varepsilon/a$ inside the $p_r$-dimensional ball
$\mathcal B_r$.  Comparing $p_r$-dimensional volumes gives at least
\[
\left(\frac{\delta a}{2\sqrt{2r}\,\varepsilon}\right)^{p_r}
\]
distinct tuples, which proves
Equation~\eqref{eq:curvature-hedging-joint}.

For achievability, the state quantizes the $p_r$ independent entries of the
first $r$ columns of $H$, respecting symmetry, on a public scalar grid of
spacing $\Delta$.  Since every relevant entry lies in $[-h_2,h_2]$,
\[
M\leq
p_r\log_2(1+2h_2/\Delta)+O(p_r).
\]
For request $q$, the quantized column has Euclidean error at most
$\sqrt d\,\Delta/2$, so the residual target lies in a $d$-dimensional ball of
radius at most $a\sqrt d\,\Delta/2$.  An $\varepsilon$-cover therefore uses
\[
B\leq
d\log_2\!\left(1+\frac{a\sqrt d\,\Delta}{\varepsilon}\right)+O(1)
\]
bits.  Choosing $\Delta=c\varepsilon^\alpha/\sqrt d$ for a constant $c$
depending only on $(m,L,a)$ gives the stated interpolation.  At $\alpha=1$,
choose the grid fine enough that the cached column alone is within
$\varepsilon$ and send no transcript.  At $\alpha=0$, omit the state and cover
the exact request-specific target ball directly.  These endpoint constructions
give exactly $B=0$ and $M=0$, respectively.

\section{Proofs of the general accounting results}\label{app:general}

\subsection{Proof of Theorem~\ref{thm:gaussian-ridge}}

Write $E_m=W_m-m\Sigma$.  Since $\|E_m\|_{\rm op}=O_p(\sqrt m)$ for fixed
$d$, the resolvent expansion around $A_0=m\Sigma+\lambda I$ gives, after
application to the fixed vector $v_Q$,
\[
\Delta_m-u_m=-A_0^{-1}E_mu_m+o_p(m^{-3/2}).
\]
Moreover, $mA_0^{-1}\to\Sigma^{-1}$ and
$mu_m\to\Sigma^{-1}v_Q$.  Hence
\[
m^{3/2}(\Delta_m-u_m)
=-\bigl(mA_0^{-1}\bigr)\bigl[m^{-1/2}E_m\bigr]\bigl(mu_m\bigr)+o_p(1).
\]
For one Gaussian row $x\sim\mathcal N(0,\Sigma)$ and
$a=\Sigma^{-1}v_Q$, Isserlis' identity gives
\[
\operatorname{Cov}\!\left(x(x^\top a)-\Sigma a\right)
=(a^\top\Sigma a)\Sigma+\Sigma aa^\top\Sigma.
\]
The multivariate central limit theorem applied to $m^{-1/2}E_ma$ and Slutsky's
theorem yield Equation~\eqref{eq:gaussian-delta}, with the fixed covariance in
Equation~\eqref{eq:gaussian-limit-covariance}.  The same expansion gives
$\operatorname{Cov}(\Delta_m\mid v_Q)=m^{-3}\mathcal V+o(m^{-3})$, proving
Equation~\eqref{eq:gaussian-covariance}.  Thus the weak limit is stated with a
fixed covariance instead of a sequence of $m$-dependent Gaussian laws.

For known $\Sigma$, the first-order Gaussian approximation
$\widetilde\Delta_m\sim\mathcal N(u_m,\overline V_m)$ has the classical
squared-error reverse-water-filling rate--distortion function.  Coordinates with
$\nu_j\leq\gamma$ are not coded and active coordinates receive distortion
$\gamma$, giving Equation~\eqref{eq:reverse-waterfill}.  If
$D\geq\sum_j\nu_j=\operatorname{tr}\overline V_m$, reproduction at the mean
already meets the target.  This is an asymptotic Gaussian benchmark for the
actual Wishart correction, not a claim that the finite-$m$ correction is exactly
Gaussian.

For the covariance-uncertain statement, let
$H=(m\Sigma+\lambda I)^{-1}$.  The spectral constraint
$\sigma_-I\preceq\Sigma\preceq\sigma_+I$ is equivalent to
\[
\frac{1}{m\sigma_++\lambda}I
\preceq H\preceq
\frac{1}{m\sigma_-+\lambda}I.
\]
Lemma~\ref{lem:spectral-image-ball} applied to $Hv_Q$ gives the exact ball and
radius in Equation~\eqref{eq:population-radius}.  Its covering number has the
usual $d\log(\rho/\varepsilon)$ high-rate coefficient and equals one once the
radius is at most $\varepsilon$.  Conditioning on exact $W_m$ removes all
remaining randomness from $\Delta_m$.

\subsection{Details for Remark~\ref{rem:rank-r-dimension}}

For one realized group request, Theorem~\ref{thm:sc-ambiguity} depends on the
deleted group only through its aggregate gradient $g_Q$.  Its quadratic
subclass therefore produces the same exact $d$-dimensional ball for any
$g_Q\neq0$, regardless of how many deleted rows were used to form that gradient
or the rank of their design matrix.  This proves the one-request coefficient
$d$.  For a public $r$-dimensional subspace with orthonormal basis
$U=[u_1,\ldots,u_r]$, the object needed to answer every future direction is
$HU$.  Symmetry imposes $u_i^THu_j=u_j^THu_i$, leaving
$dr-r(r-1)/2$ independent coordinates.  Rotating the construction of
Theorem~\ref{thm:curvature-hedging} from the first $r$ coordinate vectors to
$U$ preserves its packing and codec, proving that $p_r$ is the corresponding
state-only operator-slice coefficient.

\subsection{Proof of Proposition~\ref{thm:accounting}}

Because $\widehat Y^+$ is a function of $S$, data processing gives
\[
 I(Y^+;\widehat Y^+)\leq I(Y^+;S)\leq H(S)\leq M.
\]
Every reproduction kernel achieving distortion at most $D_+$ has mutual
information at least $\rd_{Y^+}(D_+)$ by Definition~\ref{def:budget}, which proves
\eqref{eq:ret-converse}.

For retraction, condition on the side information $(S,Q)$.  The decoder output is
a function of $(S,Q,U)$, hence
\[
 I(Y^-;\widehat Y^-\mid S,Q)
 \leq I(Y^-;U\mid S,Q)
 \leq H(U\mid S,Q)
 \leq B.
\]
Taking the infimum over all reproduction kernels with expected distortion at most
$D_-$ proves~\eqref{eq:upd-converse}.

For the sum bound, condition only on $Q$ and treat $(S,U)$ as the message:
\begin{align*}
 I(Y^-;\widehat Y^-\mid Q)
 &\leq I(Y^-;S,U\mid Q)\\
 &\leq H(S\mid Q)+H(U\mid S,Q)\\
 &\leq H(S)+H(U\mid S,Q)\\
 &\leq M+B.
\end{align*}
The conditional rate--distortion definition yields~\eqref{eq:sum-converse}.  If a
task-independent initialization $S_0$ is present, apply the same argument after
conditioning every term on $S_0$.
\subsection{Proof of Proposition~\ref{thm:request-region}}

Fix $q\in[K]$ and evaluate the protocol under the intervention that sets the
request to $q$.  This intervention changes neither the joint law of
$(\Theta_{1:K},S,R)$ nor the marginal law of $\Theta_q$.  In particular,
$\Theta_q$ remains a uniform $d$-bit source independent of the public seed
$R$.

For each realization $R=r$, let
\[
D_q(r)
=
\mathbb E\!\left[
\frac{1}{d}d_{\rm H}(\Theta_q,\widehat\Theta_q)
\,\middle|\,R=r
\right].
\]
Define the binary rate--distortion function on the full distortion range by
\[
\mathcal R_{\rm bin}(t)
=
\begin{cases}
d[1-h_2(t)], & 0\leq t\leq \tfrac12,\\
0, & \tfrac12<t\leq1.
\end{cases}
\]
For every $r$, the binary rate--distortion converse gives
\[
I(\Theta_q;\widehat\Theta_q\mid R=r)
\geq
\mathcal R_{\rm bin}(D_q(r)).
\]
The function $\mathcal R_{\rm bin}$ is convex on $[0,1]$.  Therefore, averaging
over $R$ and applying Jensen's inequality yields
\begin{align*}
I(\Theta_q;\widehat\Theta_q\mid R)
&\geq
\mathbb E_R[\mathcal R_{\rm bin}(D_q(R))]\\
&\geq
\mathcal R_{\rm bin}\!\left(\mathbb E_R[D_q(R)]\right)\\
&=
d[1-h_2(D_q)],
\end{align*}
where the final equality uses $D_q\leq\tfrac12$.
Data processing then gives
\begin{align*}
d[1-h_2(D_q)]
&\leq I(\Theta_q;\widehat\Theta_q\mid R)\\
&\leq I(\Theta_q;S,\widehat\Theta_q\mid R)\\
&=I(\Theta_q;S\mid R)
  +I(\Theta_q;\widehat\Theta_q\mid S,R)\\
&\leq I(\Theta_q;S\mid R)+H(U_q\mid S,R)\\
&\leq I(\Theta_q;S\mid R)+B_q.
\end{align*}
Thus, with $m_q=I(\Theta_q;S\mid R)$,
\[
B_q\geq[r_q-m_q]_+.
\]

It remains to bound the total information assigned by the state.  Since
$\Theta_1,\ldots,\Theta_K$ are mutually independent and independent of $R$,
\begin{align*}
\sum_{q=1}^K I(\Theta_q;S\mid R)
&\leq I(\Theta_{1:K};S\mid R)\\
&\leq H(S\mid R)\\
&\leq M.
\end{align*}
This proves the converse region in Equation~\eqref{eq:request-region}.

For achievability, fix an allocation $m_{1:K}$ satisfying
$\sum_qm_q\leq M$.  If $0\leq m_q<r_q$, choose a coarse distortion
$D_q^{(0)}\in[D_q,1/2]$ such that
\[
d[1-h_2(D_q^{(0)})]=m_q.
\]
The endpoint cases are interpreted by continuity.  A Bernoulli source under
Hamming distortion is successively refinable
\citep{equitz1991successive}; hence it admits a base description of
$m_q+o(d)$ bits and a refinement of
$r_q-m_q+o(d)$ bits that attains distortion $D_q$.  If $m_q\geq r_q$, the
state stores a complete rate--distortion description and no refinement is
needed.

Concatenating the $K$ independent base descriptions produces a state that is
formed before the request is selected.  Under the intervention that selects
$q$, the encoder sends only the refinement associated with target $\Theta_q$.
Because $K$ is fixed, the sum of the per-target coding redundancies remains
$o(d)$.  This proves asymptotic achievability of every allocation in
Equation~\eqref{eq:request-region}.

Minimizing the largest refinement rate gives
Equation~\eqref{eq:minimax-update}.  When all distortions are equal, symmetry
allocates $M/K$ state bits to each target until its rate--distortion
description is complete, yielding
\[
\left[d\{1-h_2(D)\}-\frac{M}{K}\right]_++o(d).
\]

At $D=0$, the state stores uncoded target coordinates.  Public randomization
rotates indivisible stored coordinates across the $K$ targets, giving expected
state entropy $M$ and conditional update entropy
$[d-M/K]_+$ for every request.  This attains the entropy-budget converse
exactly.  Under deterministic fixed-length budgets, an optimal integer
allocation stores at least $\lfloor M/K\rfloor$ coordinates for every target,
and therefore requires
\[
\left[d-\left\lfloor\frac{M}{K}\right\rfloor\right]_+
\]
update bits in the worst case.
\subsection{Proof of Proposition~\ref{thm:covering}}

Fix $c=(s,q)$.  A $B$-bit decoder has at most $2^B$ possible outputs.  If its
worst-case distortion is at most $D$, the closed $D$-balls centered at those
outputs cover $\Tcal_c$.  Therefore
$2^B\geq\Ncov(\Tcal_c,\dist,D)$.

Conversely, choose a minimum $D$-cover and give its centers public indices.  The
encoder maps a target to any covering center and sends its index.  A fixed-length
index uses the ceiling in~\eqref{eq:cover-exact}.  Finally, any set with pairwise
distance greater than $2D$ can contribute at most one point to a $D$-ball, so
$\Npack(\Tcal_c,\dist,2D)\leq\Ncov(\Tcal_c,\dist,D)$.

\section{Proof of Proposition~\ref{thm:ridge-states} (ridge decoder states)}
\label{app:ridge-states}

\paragraph{The downdate identity and its rank.}
Write $U=X_Q^\top\in\R^{d\times k}$, so the retained normal matrix is
$A_\lambda-UU^\top$ and the retained right-hand side is $c-Uy_Q$.  Since
\[
 A_\lambda-UU^\top=X_-^\top X_-+\lambda I_d\succ0,
\]
congruence by $A_\lambda^{-1/2}$ gives
$I_d-A_\lambda^{-1/2}UU^\top A_\lambda^{-1/2}\succ0$.  The nonzero eigenvalues
of $A_\lambda^{-1/2}UU^\top A_\lambda^{-1/2}$ equal those of
$U^\top A_\lambda^{-1}U$, so
$I_k-X_QA_\lambda^{-1}X_Q^\top\succ0$ and $M_Q$ exists.  The
Woodbury identity gives
\[
 (A_\lambda-UU^\top)^{-1}
 =A_\lambda^{-1}+A_\lambda^{-1}UM_QU^\top A_\lambda^{-1}.
\]
Applying it to $c-Uy_Q$ and using $M_QX_QA_\lambda^{-1}U=M_Q-I_k$,
\[
 w^-
 =w-A_\lambda^{-1}Uy_Q
 +A_\lambda^{-1}UM_Q\left(X_Qw-X_QA_\lambda^{-1}Uy_Q\right)
 =w+A_\lambda^{-1}UM_Q\left(X_Qw-y_Q\right),
\]
which is \eqref{eq:ridge-downdate}.  The correction is a linear combination of
the $k$ columns of $A_\lambda^{-1}X_Q^\top$, so it lies in a subspace of
dimension at most $\min(k,d)$.

\paragraph{Part (a).}
A decoder holding exact real-valued $(X_-,y_-)$ and $\lambda$ evaluates the retained normal equations directly. Under the subsection's convention, no additional transcript is required, so $B=0$ at zero distortion for either request interface.

\paragraph{Part (b).}
A decoder holding exact real-valued $(G,c)$ and, through a value-exposing request, exact $(X_Q,y_Q)$ forms $G-X_Q^\top X_Q$ and $c-X_Q^\top y_Q$, which are exactly the retained sufficient statistics, and solves. Again the additional transcript has $B=0$ under this convention. Under an identifier-only request the deleted contributions remain unknown. Their conditional family depends on the data source and on what the identifiers reveal; $(G,c)$ alone therefore yields no universal identifier-only rate. This is why part (b) is an interface-conditional zero-rate statement instead of a finite-memory comparison.

\subsection{Proof of Theorem~\ref{thm:ridge-radius} (exact target ball)}

We first isolate the geometric fact used by both ridge results.

\begin{lemma}[Image of a vector under a spectral interval]
\label{lem:spectral-image-ball}
For $x\ne0$ and $0<h_1\leq h_2$,
\begin{equation}
 \{Hx:H=H^\top,\ h_1I_d\preceq H\preceq h_2I_d\}
 =\overline B\!\left(\bar h x,\delta\lVert x\rVert_2\right),
 \quad
 \bar h=\frac{h_1+h_2}{2},\quad
 \delta=\frac{h_2-h_1}{2}.
 \label{eq:spectral-image-ball}
\end{equation}
\end{lemma}

\begin{proof}
Write $H=\bar hI_d+E$.  The spectral constraint is equivalent to $E=E^\top$
and $\lVert E\rVert_{\mathrm{op}}\leq\delta$, which implies
$\lVert Ex\rVert_2\leq\delta\lVert x\rVert_2$.  Conversely, take any $z$ in
that Euclidean ball around zero.  If $z=0$, use $E=0$.  Otherwise choose a
symmetric orthogonal Householder map $Q$ satisfying
$Qx/\lVert x\rVert_2=z/\lVert z\rVert_2$ and set
$E=(\lVert z\rVert_2/\lVert x\rVert_2)Q$.  Then $E=E^\top$,
$\lVert E\rVert_{\mathrm{op}}\leq\delta$, and $Ex=z$.
\end{proof}

For the family in Theorem~\ref{thm:ridge-radius}, the full normal matrix is
$A_\lambda=\lambda I_d+xx^\top+R$.  The labels in
\eqref{eq:ridge-radius-family} give
\[
 X^\top y
 =xy_1+R^{1/2}R^{-1/2}(A_\lambda w-xy_1)
 =A_\lambda w,
\]
so every full-data solution equals the public $w$.  After deleting $(x,y_1)$,
the retained normal equations yield
\begin{equation}
 w^-=w+\varrho(\lambda I_d+R)^{-1}x.
 \label{eq:ridge-radius-target}
\end{equation}
As $R$ ranges over
$[r_{\min}I_d,r_{\max}I_d]$ in Loewner order,
$H=(\lambda I_d+R)^{-1}$ ranges over every symmetric matrix satisfying
$h_1I_d\preceq H\preceq h_2I_d$.  Lemma~\ref{lem:spectral-image-ball}
therefore turns \eqref{eq:ridge-radius-target} into exactly the ball claimed in
Theorem~\ref{thm:ridge-radius}.

A parameters-only decoder has identical side information on every member of the
family.  Its outputs must form an $\varepsilon$-cover of the target ball.
Euclidean volume comparison gives
$\Ncov(\overline B(0,\rho),\lVert\cdot\rVert_2,\varepsilon)
\geq(\rho/\varepsilon)^d$, while the standard volumetric grid bound gives
$\Ncov\leq(1+2\rho/\varepsilon)^d$.  The one-shot characterization in
Proposition~\ref{thm:covering} proves \eqref{eq:ridge-radius-bits}.

\subsection{Proof of Theorem~\ref{thm:ridge-hedging} (request hedging)}

Let
\[
 \mathcal H=\{H=H^\top:h_1I_d\preceq H\preceq h_2I_d\}.
\]
For each $H\in\mathcal H$, set
$R=H^{-1}-(\lambda+a^2)I_d$ and
$X=[\,aI_d;R^{1/2}\,]$.  With
$y_i=aw_i-\varrho_i$ for the public deletable rows and
\[
 y_R=R^{-1/2}(H^{-1}w-ay)
\]
for the private rows, the full normal matrix is $H^{-1}$ and
$X^\top[y;y_R]=H^{-1}w$.  Hence every instance has the same deployed solution
$w$.  Deleting public row $i$ and applying the rank-one identity gives
\begin{equation}
 \Phi_i(H):=w_i^--w
 =\frac{\varrho_i a\,He_i}{1-a^2H_{ii}}.
 \label{eq:ridge-hedging-map}
\end{equation}
The denominator is positive because
$a^2h_2=a^2/(\lambda+a^2+r_{\min})<1$.

The map in \eqref{eq:ridge-hedging-map} is quantitatively invertible column by
column.  Write $\phi_{ii}=e_i^\top\Phi_i$ and
\[
 s_i=\varrho_i a+a^2\phi_{ii}
     =\frac{\varrho_i a}{1-a^2H_{ii}}.
\]
Then $He_i=\Phi_i/s_i$ and $|s_i|\geq|\varrho_i|a$.  For two admissible
matrices $H,H'$,
\begin{align}
 \lVert He_i-H'e_i\rVert_2
 &\leq \frac{1+a^2h_2}{|\varrho_i|a}
          \lVert\Phi_i(H)-\Phi_i(H')\rVert_2 \\
 &\leq \frac{2}{|\varrho_i|a}
          \lVert\Phi_i(H)-\Phi_i(H')\rVert_2.
 \label{eq:ridge-hedging-inverse-lipschitz}
\end{align}

Consider the tuple consisting of one state and the $d$ transcripts it would
send under the $d$ possible requests.  There are at most $2^{M+dB}$ such
tuples.  If two matrices induce the same tuple, the corresponding targets for
each request are within $2\varepsilon$, so
\eqref{eq:ridge-hedging-inverse-lipschitz} gives
\[
 \lVert H-H'\rVert_{\mathrm{op}}
 \leq\lVert H-H'\rVert_F
 \leq\frac{4\sqrt d\,\varepsilon}{\varrho_{\min}a}.
\]
The set $\mathcal H$ is the operator-norm ball of radius $\delta$ around
$\bar hI_d$ inside the $p$-dimensional vector space of symmetric matrices.
Lebesgue volume in that space therefore requires at least
$(\delta\varrho_{\min}a/(4\sqrt d\,\varepsilon))^p$ tuples, proving
\eqref{eq:ridge-hedging-joint}.  For a fixed request $i$, the pair $(S,U_i)$
takes at most $2^{M+B}$ values.  Lemma~\ref{lem:spectral-image-ball} shows that
$He_i$ fills a $d$-dimensional ball of radius $\delta$, and the same inverse
bound gives covering radius $4\varepsilon/(|\varrho_i|a)$.  Volume comparison
proves \eqref{eq:ridge-hedging-single}.

For achievability, quantize the $p$ upper-triangular entries of $H$ on a public
uniform grid with step $\Delta$.  This uses
\[
 M\leq p\log_2(1+2h_2/\Delta)+p
\]
bits, and every reconstructed column has error at most
$\sqrt d\,\Delta/2$.  Because
$1-a^2H_{ii}\geq1-a^2h_2>0$ on the compact family, the forward maps
$He_i\mapsto\Phi_i(H)$ are uniformly Lipschitz.  Thus the state determines the
requested correction to error at most $C\sqrt d\,\Delta$ for a finite constant
$C$ depending only on the fixed family.  The update sends the index of an
$\varepsilon$-cover point in that residual ball, using at most
\[
 d\log_2\!\left(1+\frac{2C\sqrt d\,\Delta}{\varepsilon}\right)+1
\]
bits.  Taking $\Delta=c\varepsilon^\theta/\sqrt d$ gives
\eqref{eq:ridge-hedging-ach-m}--\eqref{eq:ridge-hedging-ach-b}.  At
$\theta=1$, choose $c$ so that the state reconstruction error is at most
$\varepsilon$ and send no update.  At $\theta=0$, omit the state and cover the
compact target family directly.  These separate endpoint constructions give
exactly $B=0$ and $M=0$, respectively.

\subsection{Proof of Proposition~\ref{prop:ridge-precision}}

Let $A_-=G-xx^\top+\lambda I_d$, $E=\widetilde G-G$, and
$f=\widetilde c-c$.  The exact and perturbed retained normal equations imply
\[
 (A_-+E)(\widetilde w^--w^-)=f-Ew^-.
\]
Since $A_-\succeq\lambda I_d$ and
$\lVert E\rVert_{\mathrm{op}}\leq\eta_G<\lambda$,
$\lambda_{\min}(A_-+E)\geq\lambda-\eta_G$.  Taking norms proves
\eqref{eq:ridge-precision}.

\end{document}